\documentclass[11pt]{article} 

\usepackage{times}

\usepackage{amsmath,amsfonts,bm}

\def\figref#1{figure~\ref{#1}}

\def\eqref#1{equation~\ref{#1}}

\def\floor#1{\lfloor #1 \rfloor}
\def\1{\bm{1}}

\def\va{{\bm{a}}}
\def\vb{{\bm{b}}}
\def\vc{{\bm{c}}}

\def\vi{{\bm{i}}}
\def\vj{{\bm{j}}}

\def\vu{{\bm{u}}}
\def\vv{{\bm{v}}}
\def\vw{{\bm{w}}}
\def\vx{{\bm{x}}}
\def\vy{{\bm{y}}}
\def\vz{{\bm{z}}}

\def\mA{{\bm{A}}}
\def\mB{{\bm{B}}}

\def\mD{{\bm{D}}}

\def\mM{{\bm{M}}}

\def\mW{{\bm{W}}}

\DeclareMathAlphabet{\mathsfit}{\encodingdefault}{\sfdefault}{m}{sl}
\SetMathAlphabet{\mathsfit}{bold}{\encodingdefault}{\sfdefault}{bx}{n}
\newcommand{\tens}[1]{\bm{\mathsfit{#1}}}
\def\tA{{\tens{A}}}

\def\tW{{\tens{W}}}

\DeclareMathOperator{\sign}{sign}

\usepackage{hyperref}
\usepackage{url}

\usepackage{algorithm}
\usepackage{algorithmic}

\usepackage{amsfonts}
\usepackage{multirow}
\usepackage{amsmath}
\usepackage{graphicx}
\usepackage{amssymb}
\usepackage{dsfont}
\usepackage{afterpage}
\usepackage{wrapfig}

\usepackage[margin=1.25in]{geometry}

\usepackage[utf8]{inputenc}
\usepackage{pgfplots}
\usepackage{tikz}
\usepgfplotslibrary{groupplots}
\usetikzlibrary{decorations.pathmorphing}
\usetikzlibrary{decorations.fractals}
\usepackage{xcolor}

\newtheorem{theorem}{Theorem}

\newtheorem{lemma}{Lemma}
\newtheorem{corollary}{Corollary}
\newtheorem{assumption}{Assumption}

\newtheorem{conjecture}{Conjecture}
\newcommand{\BlackBox}{\rule{1.5ex}{1.5ex}}  
\newenvironment{proof}{\par\noindent{\bf Proof\ }}{\hfill\BlackBox\\}

\newcommand{\lemref}[1]{Lemma~\ref{#1}}
\newcommand{\thmref}[1]{Theorem~\ref{#1}}
\newcommand{\crlref}[1]{Corollary~\ref{#1}}

\newcommand{\naturals}{\mathbb{N}}
\newcommand{\reals}{\mathbb{R}}

\newcommand{\norm}[1]{\left\lVert #1 \right\rVert}
\newcommand{\mean}[2]{\mathbb{E}_{#1} \left[ #2 \right]}

\newcommand{\net}{\mathcal{N}_{\tW,\mB}}

\newcommand{\p}{^{\prime}}
\newcommand{\abs}[1]{\left| #1 \right|}

\newcommand{\todo}[1]{\textbf{TODO: #1}}

\newcommand{\given}{\Big|}

\newcommand{\prob}[2]{\mathbb{P}_{#1}\left[#2\right]}

\newcommand{\diag}{\text{diag}}

\usepackage{authblk}

\title{Is Deeper Better only when Shallow is Good?}

\author{
  Eran Malach
  }
\author{
  Shai Shalev-Shwartz
  }
\affil{School of Computer Science, The Hebrew University, Israel}
\date{}

\begin{document}

\maketitle

\begin{abstract}%
Understanding the power of depth in feed-forward neural networks
is an ongoing challenge in the field of deep learning theory.
While current works account for the importance of depth
for the expressive power of neural-networks,
it remains an open question whether these benefits
are exploited during a gradient-based optimization process.
In this work we explore the relation between
expressivity properties of deep networks and the
ability to train them efficiently using gradient-based algorithms.
We give a depth separation argument for distributions with
fractal structure, showing that they can be expressed efficiently
by deep networks, but not with shallow ones.
These distributions have a natural coarse-to-fine structure,
and we show that the balance between the coarse and fine details has a crucial
effect on whether the optimization process is likely to succeed.
We prove that when the distribution is concentrated on the fine details,
gradient-based algorithms are likely to fail.
Using this result we prove that,
at least in some distributions, the success of learning deep
networks depends on whether the distribution can be well approximated
by shallower networks, and we conjecture that this property holds in general.
\end{abstract}

\section{Introduction}
A fundamental question in studying deep networks is understanding why
and when ``deeper is better''. There have been several results
identifying a ``depth separation'' property: showing that there
exist functions which are realized by deep networks of moderate width,
that cannot be approximated by shallow networks,
unless an exponential number of units is used.  However, this is
unsatisfactory, as the fact that a certain network architecture can
express some function does not mean that we can learn this function
from training data in a reasonable amount of training time. In fact,
there is theoretical evidence showing that gradient-based algorithms
can only learn a small fraction of the functions that are expressed
by a given neural-network (e.g \cite{shalev2017failures}).

This paper relates expressivity properties of deep networks to the
ability to train them efficiently using a gradient-based algorithm.
We start by giving depth separation arguments for distributions with
fractal structure. In particular, we show that deep networks are able
to exploit the self-similarity property of fractal distributions, and
thus realize such distributions with a small number of parameters.  On
the other hand, we show that shallow networks need a number of
parameters that grows exponentially with the intrinsic ``depth'' of
the fractal.  The advantage of fractal distributions is that they
exhibit a clear coarse-to-fine structure. We show that if the
distribution is more concentrated on the ``coarse'' details of the
fractal, then even though shallower networks cannot exactly express
the underlying distribution, they can still achieve a good
approximation.  We introduce the notion of \textbf{approximation
curve}, that characterizes how the examples are distributed between
the ``coarse'' details and the ``fine'' details of the fractal.  The
approximation curve captures the relation between the growth in the
network's depth and the improvement in approximation.

We next go beyond pure expressivity analysis, and claim that the
approximation curve plays a key role not only in approximation
analysis, but also in predicting the success of gradient-based
optimization algorithms. Specifically, we show that if the
distribution is concentrated on the ``fine'' details of the fractal,
then gradient-based optimization algorithms are likely to fail. In
other words, the ``stronger'' the depth separation is (in the sense that shallow
networks cannot even approximate the distribution) the harder it is to
learn a deep network with a gradient-based algorithm. While we prove this statement for a
specific fractal distribution, we state a conjecture aiming at
formalizing this statement in a more general sense. Namely, we
conjecture that a distribution which cannot be approximated by a
shallow network cannot be learned using gradient-based algorithm, even
when using a deep architecture. 
We perform experiments on learning fractal
distributions with deep networks trained with SGD and assert that the
approximation curve has a crucial effect on whether a depth efficiency
is observed or not.  These results provide new insights as to when
such deep distributions can be learned.

Admittedly, this paper is focused on analyzing a family of
distributions that is synthetic by nature.
That said, we note that the conclusions from this analysis
may be interesting for the broader effort of understanding
the power of depth in neural-networks.
As mentioned, we show that there exist distributions
with depth separation property (that are expressed efficiently
with deep networks but not with shallow ones), that cannot be learned
by gradient-based optimization algorithms.
This result implies that any depth separation argument that does
not consider the optimization process should be taken with a grain of salt.
Additionally, our results hint that the success of learning deep
networks depends on whether the distribution can be approximated
by shallower networks. Indeed, this property is often observed
in real-world distributions, where deeper networks perform better,
but shallower networks exhibit good (if not perfect) performance.

\section{Related Work}

In recent years there has been a large number of works studying the
expressive power of deep and shallow networks.
The main goal of this research direction is to show families of functions or distributions that are realizable with deep networks
of modest width,
but require exponential number of neurons to approximate by shallow networks. We refer to such results as depth separation results.

Many of these works consider various measures of ``complexity'' that 
grow exponentially fast with the depth of the network, but not with
the width. Hence, such measures provide a clear
separation between deep and shallow networks.
For example, the works of \cite{pascanu2013number, montufar2014number,
montufar2017notes, serra2017bounding}
show that the number of linear regions grows exponentially with
the depth of the network, but only polynomially with the width.
The work of \cite{poole2016exponential} shows that for networks
with random weights, the curvature of the functions calculated by
the networks grows exponentially with depth but not with width.
In another work, \cite{raghu2016expressive} shows that the trajectory
length, which measures the change in the output along a one-dimensional
path, grows exponentially with depth.
Finally, the work of \cite{telgarsky2016benefits} utilizes
the number of oscillations in the function to give a depth
separation result.

While such works give general characteristics of function families,
they take a seemingly worst-case approach. Namely,
these works show that there exist functions implemented
by deep networks that are hard to approximate with a shallow net.
But as in any worst-case analysis, it is not clear whether
such analysis applies to the typical cases encountered in the
practice of neural-networks. 
In order to answer this concern, recent works show depth separation results
for narrower families of functions that appear simple or ``natural''.
For example, the work of \cite{telgarsky2015representation} shows
a very simple construction of a function on the real line that
exhibits a depth separation property.
The works of \cite{eldan2016power, safran2016depth} show
a depth separation argument for very natural functions, like
the indicator function of the unit ball.
The work of \cite{daniely2017depth} gives similar results for a richer
family of functions.
 Another series of works
by \cite{mhaskar2016deep, poggio2017and} show that compositional
functions, namely functions of functions, can be well approximated 
by deep networks. The works of \cite{delalleau2011shallow,cohen2016expressive} show that
compositional properties establish depth separation for sum-product
networks.

Our work shares similar motivations with the above works. Namely, our goal
is to construct a family of distributions that demonstrate the power
of deep networks over shallow ones. Unlike these works, we do not
limit ourselves to expressivity results alone, but rather take another step
into exploring whether these distributions can be learned by
gradient-based optimization algorithms.

Finally, while we are not aware of any work that directly considers
fractal structures in the context of deep learning,
there are a few works that tie them to other fields in machine learning.
Notably, \cite{bloem2010fractal} gives a thorough review of
fractal geometry from a machine learning perspective, suggesting
that their structure can be exploited in various machine learning tasks.
The work of \cite{korn2001dimensionality} considers the effect
of fractal structure on the performance of nearest neighbors algorithms.
We also note that fractal structures are exploited in image
compression (refer to \cite{barnsley1993fractal} for a review).
These works mainly give motivation to look at fractal geometry
in the context of deep learning, as these seem relevant
for other problems in machine learning.

\section{Preliminaries}

Let $\mathcal{X} = \reals^d$ be the domain space and
$\mathcal{Y} = \{\pm 1\}$ be the label space.
We consider distributions defined over sets generated
by an iterated function system (IFS).
An IFS is a method for constructing fractals, where a finite set of contraction mappings are applied iteratively, starting with some
arbitrary initial set. Applying such process ad infinitum generates a
self-similar fractal. In this work we will consider sets generated
by performing a finite number of iterations from such process.
We refer to the number of iterations of the IFS as the ``depth''
of the generated set.

\setlength{\intextsep}{0pt}%
\begin{wrapfigure}{r}{0.5\textwidth}
\begin{center}
\begin{tikzpicture}[shorten >=1pt,->]
  \tikzstyle{bigrect}=[rectangle,draw=black,fill=black!25,minimum size=30pt,inner sep=0pt]
  \tikzstyle{whiterect}=[rectangle,fill=white,draw=black,minimum size=30pt]
  \tikzstyle{midrect}=[rectangle,fill=black!25,minimum size=10pt,inner sep=0pt]
  \tikzstyle{whitemidrect}=[rectangle,draw=black,fill=white,minimum size=10pt,inner sep=0pt]
  \tikzstyle{smallrect}=[rectangle,fill=black!25,minimum size=3.33pt,inner sep=0pt]

    \node[bigrect] (K-0) at (0,0) {$K_0$};
    
    \node[whiterect] (K-1) at (1.83333,0) {$K_1$};
    \node[midrect] (K-1-1) at (1.5,0.33) {};
    \node[midrect] (K-1-2) at (1.5,-0.33) {};
    \node[midrect] (K-1-3) at (2.16,0.33) {};
    \node[midrect] (K-1-4) at (2.16,-0.33) {};
    
    \draw (K-0)  to [out=15,in=180] node[above]{$F_1$} (K-1-1);
    \draw (K-0)  to [out=-15,in=180] node[below]{$F_2$} (K-1-2);
    \draw (K-0)  to [out=60,in=150] node[above]{$F_3$} (K-1-3);
    \draw (K-0)  to [out=-60,in=-150] node[below]{$F_4$} (K-1-4); 

    \node[whitemidrect] (K-2-1) at (3.165,0.33) {};
    \node[whitemidrect] (K-2-2) at (3.165,-0.33) {};
    \node[whitemidrect] (K-2-3) at (3.825,0.33) {};
    \node[whitemidrect] (K-2-4) at (3.825,-0.33) {};

    \node[smallrect] (K-2-1-1) at (3.066,0.429) {};
    \node[smallrect] (K-2-1-2) at (3.066,0.231) {};
    \node[smallrect] (K-2-1-3) at (3.264,0.429) {};
    \node[smallrect] (K-2-1-4) at (3.264,0.231) {};
   
    \node[smallrect] (K-2-2-1) at (3.726,0.429) {};
    \node[smallrect] (K-2-2-2) at (3.726,0.231) {};
    \node[smallrect] (K-2-2-3) at (3.934,0.429) {};
    \node[smallrect] (K-2-2-4) at (3.934,0.231) {};
    
    \node[smallrect] (K-2-3-1) at (3.066,-0.429) {};
    \node[smallrect] (K-2-3-2) at (3.066,-0.231) {};
    \node[smallrect] (K-2-3-3) at (3.264,-0.429) {};
    \node[smallrect] (K-2-3-4) at (3.264,-0.231) {};
   
    \node[smallrect] (K-2-4-1) at (3.726,-0.429) {};
    \node[smallrect] (K-2-4-2) at (3.726,-0.231) {};
    \node[smallrect] (K-2-4-3) at (3.934,-0.429) {};
    \node[smallrect] (K-2-4-4) at (3.934,-0.231) {};
    
    \draw (K-1)  to [out=15,in=180] node[above]{$F_1$} (K-2-1);
    \draw (K-1)  to [out=-15,in=180] node[below]{$F_2$} (K-2-2);
    \draw (K-1)  to [out=60,in=120] node[above]{$F_3$} (K-2-3);
    \draw (K-1)  to [out=-60,in=-120] node[below]{$F_4$} (K-2-4);
    
    \node at (3.5,0) {$K_2$};

\end{tikzpicture} \\
\includegraphics[trim={3cm 1.5cm 3cm 1.5cm},clip,scale=0.15]{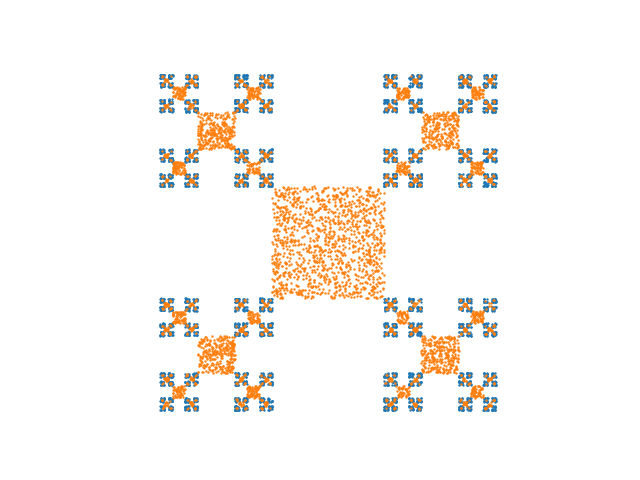}
\includegraphics[trim={3cm 1.5cm 3cm 1.5cm},clip,scale=0.15]{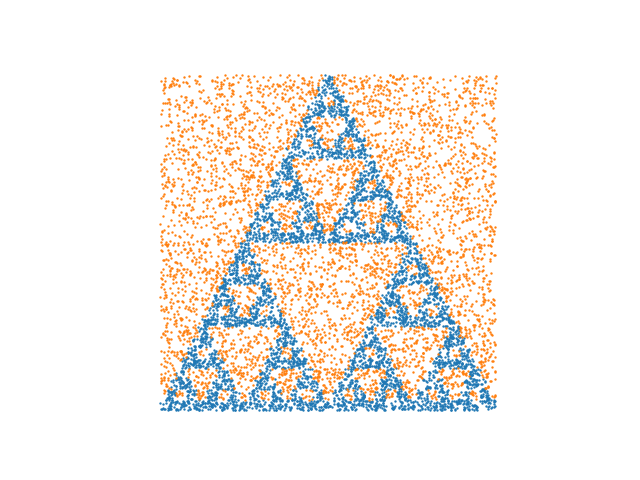}
\includegraphics[trim={3cm 1.5cm 3cm 1.5cm},clip,scale=0.15]{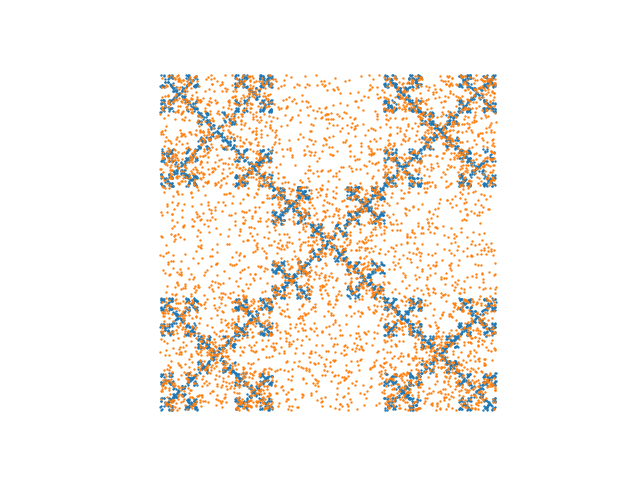}
\includegraphics[trim={3cm 1.5cm 3cm 1.5cm},clip,scale=0.15]{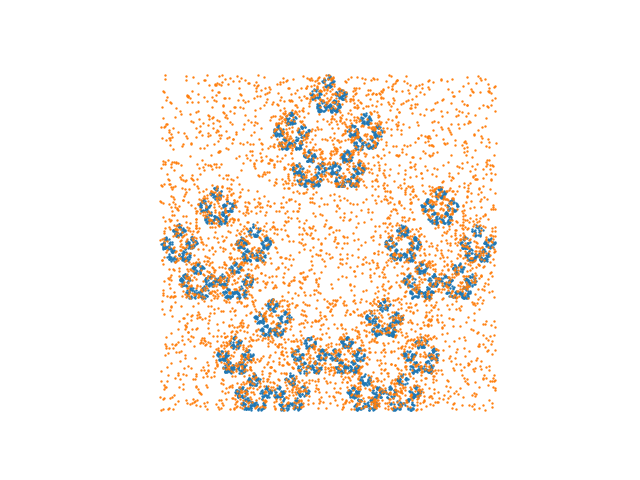}
\caption{IFS and fractal distributions.}
\label{fig:ifs}
\end{center}
\vskip -1in
\end{wrapfigure}
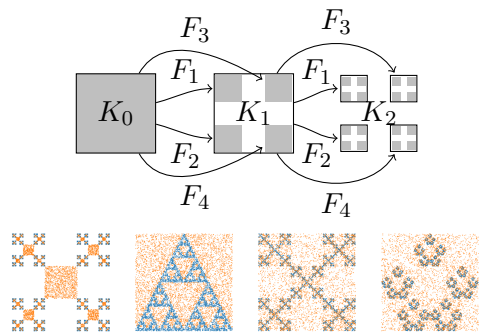

Formally, an IFS is defined by a set of $r$ contractive affine
\footnote{In general, IFSs can be constructed with non-linear transformations, but in this paper we discuss only affine IFS.} transformations $F = (F_1, \dots, F_r)$, where
$F_i(\vx) = \mM^{(i)} \vx + \vv^{(i)}$
with full-rank matrix $\mM^{(i)} \in \reals^{d \times d}$,
vector $\vv^{(i)} \in \reals^d$,
s.t $\norm{F_i(\vx)-F_i(\vy)} < \norm{\vx-\vy}$ for all $\vx,\vy \in \mathcal{X}$
(we use $\norm{\cdot}$ to denote the $\ell_2$ norm, unless stated otherwise).
We define the set $K_n \subseteq \mathcal{X}$ recursively by:
\begin{itemize}
\item $K_0 = [-1,1]^d$
\item $K_n = F_1(K_{n-1}) \cup \dots \cup F_r(K_{n-1})$
\end{itemize}
The IFS construction is shown in \figref{fig:ifs}.	

We define a ``fractal distributions'', denoted $\mathcal{D}_n$, to be
any balanced distribution over $\mathcal{X} \times \mathcal{Y}$ such
that positive examples are sampled from the set $K_n$ and negative
examples are sampled from its complement. Formally,
$\mathcal{D}_n = \frac{1}{2}(\mathcal{D}_n^+ + \mathcal{D}_n^-)$ where
$\mathcal{D}_n^{+}$ is a distribution over
$\mathcal{X} \times \mathcal{Y}$ that is supported on
$K_n \times \{+1\}$, and $\mathcal{D}_n^{-}$ is a distribution over
$\mathcal{X} \times \mathcal{Y}$ that is supported on
$(\mathcal{X} \setminus K_n) \times \{-1\}$. Examples for such distributions
are given in \figref{fig:ifs} and \figref{fig:cantor_data}. 

In this paper we consider
the problem of learning fractal distributions with feed-forward
neural-networks equipped with the ReLU activation.
A ReLU neural-network $\mathcal{N}_{\tW, \mB}: \mathcal{X} \to \mathcal{Y}$
of depth $t$ and width $k$ is a function defined recursively
such that $\vx^{(t)} : = \mathcal{N}_{\tW, \mB} (\vx)$, and:
\begin{enumerate}
\item $\vx^{(0)} = \vx$
\item $\vx^{(t\p)} = \sigma(\mW^{(t\p)} \vx^{(t\p-1)} + \vb^{(t\p)})$
for every $t\p \in \{1,\ldots,t-1\}$
\item $\vx^{(t)} = \mW^{(t)} \vx^{(t-1)} + \vb^{(t)}$
\end{enumerate}
Where $\mW^{(1)} \in \reals^{k \times d}, \mW^{(2)}, \dots, \mW^{(t-1)} \in \reals^{k \times k}, \mW^{(t)} \in \reals^{1 \times k}$,
$\vb^{(1)}, \dots, \vb^{(t-1)} \in \reals^k, \vb^{(t)} \in \reals$,
and $\sigma(\vx) := \max (\vx,0)$.

We denote by $\mathcal{H}_{k,t}$ the family of all functions that are
implemented by a neural-network of width $k$ and depth $t$.
Given a distribution $\mathcal{D}$ over $\mathcal{X} \times \mathcal{Y}$,
we denote the error of a network $h \in \mathcal{H}_{k,t}$ on distribution 
$\mathcal{D}$ to be $L_{\mathcal{D}}(h) :=\prob{(\vx,y) \sim \mathcal{D}_n}
{\sign(h(\vx)) \ne y}$. We denote the approximation error of $\mathcal{H}_{k,t}$
on $\mathcal{D}$ to be the minimal error of any such function:
$L_{\mathcal{D}}(\mathcal{H}_{k,t}) := \min_{h \in \mathcal{H}_{k,t}}
L_{\mathcal{D}}(h)$.

\section{Expressivity and Approximation}
In this section we analyze the expressive power of deep and shallow
neural-networks w.r.t fractal distributions.  We show two results.
The first is a depth separation property of
neural-networks. Namely, we show that shallow networks need an
exponential number of neurons to realize such distributions, while
deep networks need only a number of neurons that is linear in the
problem's parameters.  The second result bounds the approximation
error achieved by networks that are not deep enough to achieve zero
error. This bound depends on the specific properties of the fractal
distribution.

We analyze IFSs where the images of the initial set $K_0$
under the different mappings do not overlap. This property allows the
neural-network to ``reverse'' the process that generates the fractal
structure.
Additionally, we assume that the images of $K_0$ (and therefore
the entire fractal), are contained in $K_0$, which means that the
fractal does not grow in size.
This is a technical requirement that could be achieved by correctly
scaling the fractal at each step.
While these requirements are not generally assumed in the context of IFSs,
they hold for many common fractals (cantor set, sierpinsky triangle and more). Formally, we assume the following:
\begin{assumption}
\label{asm:disjoint}
There exists $\epsilon > 0$ such that
for $i \ne j \in [r]$ it holds that
$d(F_i(K_0), F_j(K_0)) > \epsilon$,
where $d(A,B) = \min_{\vx \in A, \vy \in B} \norm{\vx-\vy}$.
\end{assumption}
\begin{assumption}
\label{asm:contraction}
For each $i \in [r]$ it holds that $F_i(K_0) \subseteq K_0$.
\end{assumption}

Finally, as in many other problems in machine learning, we assume the
positive and negative examples are separated by some
margin. Specifically, we assume that the positive examples are sampled
from strictly inside the set $K_n$, with margin $\gamma$ from the set
boundary. Formally, for some set $A$, we define $A^\gamma$ to be the
set of all points that are far from the boundary of $A$ by at least
$\gamma$: $A^\gamma := \{\vx \in A ~:~ B_\gamma (\vx) \subseteq A\}$,
where $B_\gamma(\vx)$ denotes a ball around $\vx$ with radius
$\gamma$.  So our assumption is the following:
\begin{assumption}
There exists $\gamma > 0$ such that $\mathcal{D}_n^+$ is supported on
$K_n^\gamma \times \{+1\}$.
\end{assumption}

\subsection{Depth Separation}
We show that neural-networks with depth linear in $n$
(where $n$ is the ``depth'' of the fractal) can achieve zero error
on any fractal distribution satisfying the above assumptions,
with only linear width.
On the other hand, a shallow network needs a width
exponential in $n$ to achieve zero error on such distributions.

To separate such fractal distributions, we start with the following:
\begin{theorem}
\label{thm:depth_expressivity}
There exists a neural-network of width $5dr$ and depth $2n+1$,
s.t $\sign(\mathcal{N}_{\tW, \mB}(K_n^\gamma)) = 1$ and 
$\sign(\mathcal{N}_{\tW, \mB}(\mathcal{X} \setminus K_n)) = -1$.
\end{theorem}

Since, by assumption, there are no examples in the margin area, we
immediately get an expressivity result under any fractal distribution:

\begin{corollary}
\label{crl:depth_expressivity}
For any distribution $\mathcal{D}_n$ there exist neural-network of width $5dr$
and depth $2n+1$, such that $L_{\mathcal{D}_n}(\net) = 0$.
\end{corollary}

We defer the proof of \thmref{thm:depth_expressivity} to the appendix, and 
give here an intuition of how deep networks can express these seemingly complex distributions with a small number of parameters.
Note that by definition, the set $K_n$ is composed of $r$ copies
of the set $K_{n-1}$, mapped by different affine transformations.
In our construction,
each block of the network folds the different copies of $K_{n-1}$
on-top of each other, while ``throwing
away'' the rest of the examples (by mapping them to a distinct value).
The next block can then perform the same thing on all copies of $K_{n-1}$
together, instead of decomposing each subset separately.
This allows a very efficient utilization of the network parameters.

The above results show that deep networks are very efficient
in utilizing the parameters of the network, requiring a number of
parameters that grows linearly with $r$, $d$ and $n$.
Now, we want to consider the case of shallower networks,
when the depth is not large enough to achieve zero error
with linear width. Specifically, we show that when decreasing
the depth of the network by a factor of $s$, we can achieve
zero error by allowing the width to grow like $r^s$.

Notice that for any $s$ that divides $n$, any IFS of depth
$n$ with $r$ transformations can be written as depth
$\frac{n}{s}$ IFS with $r^s$ transformations. Indeed,
for $\vi = (i_1, \dots, i_s) \in [r]^s$
denote $F_{\vi} (\vx) = F_{i_1} \circ \dots \circ F_{i_s}(\vx)$,
and we have: $K_s = \cup_{\vi \in [r]^s} F_\vi (K_0)$.
So we can write a new IFS with transformations $\{F_\vi\}_{\vi \in [r]^s}$,
and these will generate $K_n$ in $\frac{n}{s}$ iterations.
This gives us a stronger version of the previous result,
which explicitly shows the trade-off between linear growth of depth
and exponential growth of width:

\begin{corollary}
\label{crl:depth_expressivity_multi}
For any distribution
$\mathcal{D}_n$ and every natural $s \le n$
there exists a neural-network of width $5dr^s$ and depth
$2 \floor{n/s}+2$,
such that $L_{\mathcal{D}_n}(\net) = 0$.
\end{corollary}

This is an upper bound on the required width of a network that can realize
$\mathcal{D}_n$, for any given depth.
To show the depth separation property, we show that a shallow network
needs an exponential number of neurons to implement the indicator function.
This gives the equivalent lower bound on the required width.

\begin{theorem}
\label{thm:shallow_expressivity}
Let $\mathcal{N}_{\tW, \mB}$ be a network of depth $t$ and of width
$k$, such that $\sign(\mathcal{N}_{\tW, \mB}(K_n^\gamma)) = 1$ and 
$\sign(\mathcal{N}_{\tW, \mB}(\mathcal{X} \setminus K_n)) = -1$.
Denote $s$ to be the ratio between the depth of the fractal and the depth
of the network, so $s := n/t$. Then the width of the network
grows exponentially with $s$, namely: $k \ge \frac{d}{e} r^{s/d}$.
\end{theorem}
\begin{proof}
From Proposition 3 in \cite{montufar2017notes} we get that
there are $\prod_{t\p=1}^t\sum_{j=0}^d \binom{k}{j} \le (ek/d)^{td}$
linear regions in $\mathcal{N}_{\tW, \mB}$ (where we use Lemma A.5 from \cite{shalev2014understanding}). Furthermore, every such linear region is an intersection
of affine half-spaces.

Note that any function such that
$\sign(f(K_n^\gamma)) = 1$ and $\sign(f(\mathcal{X} \setminus K_n)) = -1$
has at least $r^n$ such linear regions.
Indeed, notice that $K_n = \cup_{\vi \in [r]^n} F_{\vi}(K_0)$.
Assume by contradiction that there are $< r^n$ linear regions,
so there exists $\vi \ne \vj \in [r]^n$ such that
$F_{\vi}(K_0), F_{\vj}(K_0)$ are in the same linear region.
Fix $\vx \in F_{\vi}(K_0)^\gamma, \vy \in F_{\vj}(K_0)^\gamma$
and observe the function $f$ along the line from $\vx$ to $\vy$.
By our assumption $f(\vx) \ge 0$, $f(\vy) \ge 0$.
This line must cross $\mathcal{X} \setminus K_n$,
since from Assumption \ref{asm:disjoint} we get that
$d(F_{\vi}(K_0), F_{\vj}(K_0)) > 0$ for every $\vi \ne \vj \in [r]^n$.
Therefore $f$ must get negative values along the line
between $\vx$ to $\vy$, so it must cross zero at least twice.
Every linear region is an intersection
of half-spaces, and hence convex, so $f$ is linear on this path,
and we reach a contradiction.

Therefore, we get that $(ek/d)^{td} \ge r^n$, and therefore:
$k \ge \frac{d}{e} r^{s/d}$.
\end{proof}

%

This result implies that there are many
fractal distributions for which a shallow neural-network 
needs exponentially many neurons to achieve zero error on.
In fact, we show this for any distribution
without ``holes'' (areas of non-zero volume
with no examples from $\mathcal{D}_n$, outside the margin area):
\begin{corollary}
Let $\mathcal{D}_n$ be some fractal distribution, s.t for every
ball $B \subseteq K_n^\gamma \cup (\mathcal{X} \setminus K_n)$
it holds that $\prob{(\vx,y) \sim \mathcal{D}_n}{\vx \in B} > 0$.
Then for every depth $t$ and width $k$,
s.t $k < \frac{d}{e} r^{\frac{n}{td}}$, we have
$L_{\mathcal{D}_n}(\mathcal{H}_{k,t}) > 0$.
\end{corollary}

\begin{proof}
Let $\net \in \mathcal{H}_{k,t}$. From \thmref{thm:shallow_expressivity}
there exists $\vx \in K_n^\gamma$
with $\sign(\mathcal{N}_{\tW, \mB}(\vx)) = -1$
or otherwise there exists $\vx \in \mathcal{X} \setminus K_n$
with $\sign(\mathcal{N}_{\tW, \mB}(\vx)) = 1$.
Assume w.l.o.g that we have $\vx \in K_n^\gamma$
with $\sign(\mathcal{N}_{\tW, \mB}(\vx)) = -1$.
Since $\mathcal{N}_{\tW, \mB}$ is continuous,
there exists a ball around $\vx$, with $\vx \in B \subseteq K_n^\gamma$,
such that $\sign(\mathcal{N}_{\tW, \mB}(B)) = -1$.
From the properties of the distribution we get:
\begin{align*}
\prob{(\vx,y) \sim \mathcal{D}_n}
{\sign(\mathcal{N}_{\tW, \mB}(\vx)) \ne y}
&\ge \prob{(\vx,y) \sim \mathcal{D}_n}
{\sign(\mathcal{N}_{\tW, \mB}(\vx)) \ne y ~and~ \vx \in B} \\
&=\prob{(\vx,y) \sim \mathcal{D}_n}{\vx \in B} > 0
\end{align*}
\end{proof}

The previous result shows that in many cases we cannot
guarantee exact realization of ``deep'' distributions by shallow networks
that are not exponentially wide.
On the other hand, we show that in
some cases we may be able to give good guarantees
on \emph{approximating} such distributions with shallow networks,
when we take into account
how the examples are distributed within the fractal structure.
We will formalize this notion in the next part of this section.

\subsection{Approximation Curve}

Given distribution $\mathcal{D}_n$, we define the \textbf{approximation curve}
of this distribution to be the function $P: [n] \rightarrow [0,1]$,
where: 
\[
P(j) = \prob{(\vx,y) \sim \mathcal{D}_n}{\vx \notin K_j ~or~ y = 1}
\]
Notice that $P(0) = \frac{1}{2}$, $P(n) = 1$,
and that $P$ is 
non-decreasing. The approximation curve $P$ captures exactly how
the negative examples are distributed between the different
levels of the fractal structure. If $P$ grows fast at the beginning,
then the distribution is more concentrated on the low levels of
the fractal (coarse details). If $P$ stays flat until the end,
then most of the weight is on the high levels (fine details).
Figure \ref{fig:cantor_data} shows samples from two
distributions over the same fractal structure,
with different approximation curves.

\begin{figure}[t]
\begin{center}
\centerline{\includegraphics[trim={2.03cm 1.32cm 1.5cm 1.32cm},clip,width=0.3\columnwidth]{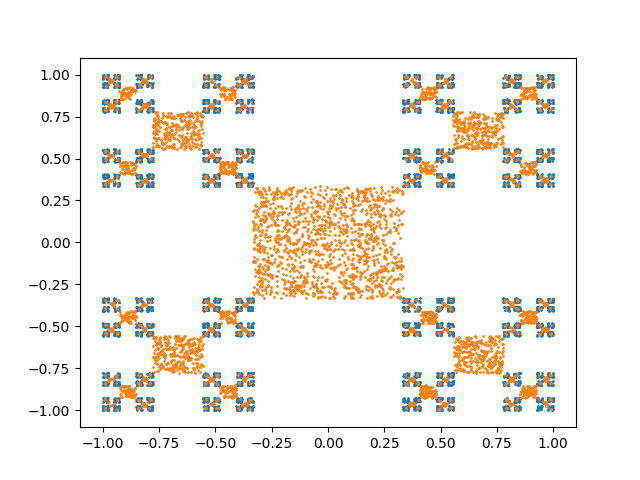}
\hspace{0.1\columnwidth}
\includegraphics[trim={2.03cm 1.32cm 1.5cm 1.32cm},clip,width=0.3\columnwidth]{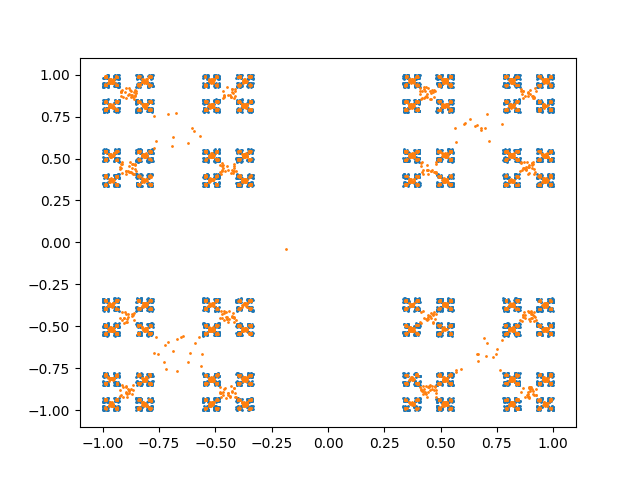}
}
\caption{2D cantor distributions of depth 5,
negative examples in orange and positive in blue.
The negative examples are concentrated
in the middle rectangle, and not in all $\mathcal{X} \setminus K_n$.
Left: ``coarse'' approximation curve (curve\#1). Right:
``fine'' approximation curve (curve\#4).}
\label{fig:cantor_data}
\end{center}
\vskip -0.2in
\end{figure}

A simple argument shows that distributions concentrated on
coarse details can be well approximated by shallower networks.
The following theorem characterizes the relation between
the approximation curve and the ``approximability'' by networks
of growing depth: 

\begin{theorem}
\label{thm:approximation_upper_bound}
Let $\mathcal{D}_n$ be some fractal distribution with approximation curve $P$.
Fix some $j,s$, then for $\mathcal{H}_{k,t}$ with depth $t=2\floor{j/s}+2$
and width $k=5dr^s$,
we have: $L_{\mathcal{D}_n}(\mathcal{H}_{k,t}) \le 1- P(j)$.
\end{theorem}

\begin{proof}
From \thmref{thm:depth_expressivity} and \crlref{crl:depth_expressivity_multi},
there exists a network of depth
$t = 2\floor{j/s}+2$ and width $5dr^s$ such that 
$\sign(\mathcal{N}_{\tW, \mB}(K_j^{\gamma})) = 1$
and $\sign(\mathcal{N}_{\tW, \mB}(\mathcal{X} \setminus K_j)) = -1$.
Notice that since $K_n^\gamma \subseteq K_j^\gamma$, we have:
$\prob{(\vx,y) \sim \mathcal{D}_n}
{\vx \notin K_j^\gamma ~and~ y = 1} = 0$.
Therefore for this network we get:
$\prob{(\vx,y) \sim \mathcal{D}_n}
{\sign(\mathcal{N}_{\tW, \mB}(\vx)) \ne y}
\le \prob{(\vx,y) \sim \mathcal{D}_n}
{x \in K_j ~and~ y \ne 1} = 1 - P(j)$.
\end{proof}

This shows that using the approximation curve of distribution $\mathcal{D}_n$
allows us to give an upper bound on the approximation error
for networks that are not deep enough.
We give a lower bound for this error in a more restricted case.
We limit ourselves to the case where $d=1$, and observe networks
of width $k<r^s$ for some $s$.
Furthermore, we assume that the probability of seeing each
subset of the fractal is the same.
Then we get the following theorem:
\begin{theorem}
\label{thm:approximation_lower_bound}
Assume that $\mathcal{D}_n$ is a distribution on $\reals$ ($d=1$).
Note that for every $j$, $K_j$ is a union of $r^{j}$ intervals,
and we denote $K_j = \cup_{i=1}^{r^j} I_i$ for intervals $I_i$.
Assume that the distribution over each interval is equal,
so for every $i,\ell,y\p$:
$\prob{(x,y) \sim \mathcal{D}_n}{x \in I_i ~and~ y = y\p}
=\prob{(x,y) \sim \mathcal{D}_n}{x \in I_\ell ~and~ y = y\p}$.
Then for depth $t$ and width $k < r^s$,
for $n > j > st$ we get:
$L_{\mathcal{D}_n}(\mathcal{H}_{k,t}) \ge (1-r^{st-j})(1-P(j))$.
\end{theorem}

The above theorem shows that for shallow networks, for which $st \ll j$,
the approximation curve gives a very tight lower bound on the approximation
error. This is due to the fact that shallow networks have a limited number
of linear regions, and hence effectively give constant prediction
on most of the ``finer'' details of the fractal distribution.
This result implies that there are fractal distributions that are not
only hard to realize by shallow networks, but that are even hard to approximate. Indeed,
fix some small $\epsilon > 0$ and let $j := st + \log_r(\frac{1}{2\epsilon})$.
Then if the approximation curve stays flat for the first $j$
levels (i.e $P(j) = \frac{1}{2}$), then from \thmref{thm:approximation_lower_bound}
the approximation error is at least $\frac{1}{2} - \epsilon$.

This gives a \textbf{strong depth separation} result: 
shallow networks have an error of
$\approx \frac{1}{2}$ while a network of depth $t \ge 2\floor{n/s}+2$
can achieve zero error (on any fractal distribution).
This strong depth separation result occurs when the distribution is concentrated on the ``fine''
details, i.e when the approximation curve stays flat throughout the
``coarse'' levels.
In the next section we
relate the approximation curve to the success of fitting a deep
network to the fractal distribution, using gradient-based optimization algorithms.
Specifically, we claim that distributions with strong depth separation
\textbf{cannot} be learned by any network, deep or shallow,
using gradient-based algorithms.

\section{Optimization Analysis}
So far, we analyzed the ability of neural-networks to express
and approximate different fractal distributions. But it remains unclear
whether these networks can be learned with gradient-based optimization
algorithms.
In this section, we show that the success of the optimization
highly depends on the approximation curve of the fractal distribution.
Namely, we show that for distributions with a ``fine'' approximation curve,
that are concentrated on the ``fine'' details of the fractal,
the optimization fails with high probability, for \textbf{any} gradient-based optimization algorithm.

To simplify the analysis,
we focus in this section on a very simple fractal distribution: a distribution over the Cantor set in $\reals$.
We begin by defining the standard construction of the Cantor set,
using an IFS. We construct the set $C_n$ recursively:
\begin{enumerate}
\item $C_0 = [0,1]$
\item $C_n = F_1(C_{n-1}) \cup F_2(C_{n-1})$
\end{enumerate}
where $F_1(x) = \frac{1}{3}-\frac{1}{3} x$ and
$F_2(x) = \frac{1}{3}+\frac{1}{3} x$.

Now, fix margin $\gamma < \frac{3^{-n}}{2}$.
We define the distribution $\mathcal{D}_n^+$ to be the uniform
distribution over $C_n^\gamma \times \{+1\}$.
The distribution $\mathcal{D}_n^-$
is a distribution over $C_0 \setminus C_n$, where we sample
from each ``level'' $C_{j}$ ($j < n$) with probability $p_j$.
Formally, we define $E_j := C_{j-1} \setminus C_{j}$
to be the $j$-th level of the negative distribution.
We use $\mathcal{U}(E_j)$ to denote the uniform distribution on set $E_j$,
then:
$\mathcal{D}_{n}^- = \sum_{j=1}^{n} p_j 
\left( \mathcal{U}(E_j) \times \{-1\} \right)$.
Notice that the approximation curve of this distribution is given by:
$P(j) = \frac{1}{2} + \frac{1}{2}\sum_{i=1}^{j} p_i$.
As before, we wish to learn
$\mathcal{D}_n = \frac{1}{2}(\mathcal{D}_n^+ + \mathcal{D}_n^-)$.
Figure \ref{fig:cantor_intervals} shows a construction of such distribution.

\subsection{Hardness of Optimization}
The main theorem in this section shows the connection between the approximation
curve and the behavior of a gradient-based optimization algorithm.
This result shows that for deep enough cantor distributions,
the value of the approximation curve on the fine details of the fractal bounds
the norm of the population gradient for randomly initialized network:
\begin{theorem}
\label{thm:hardness}
Fix some depth $t$, width $k$ and some $\delta \in (0,1)$.
Let $n, n\p \in \naturals$ such that
$n > n\p > \log^{-1}(\frac{3}{2})\log(\frac{4tk^2}{\delta})$.
Let $\mathcal{D}_n$ be some cantor distribution with approximation curve
$P$.
Assume we initialize a neural-network $\mathcal{N}_{\tW, \mB}$
of depth $t$ and width $k$,
with weights initialized uniformly in
$[-\frac{1}{2n_{in}},\frac{1}{2n_{in}}]$ (where $n_{in}$ denotes the 
in-degree of each neuron),
and biases initialized
with a fixed value $b = \frac{1}{2}$
\footnote{We note that it is standard practice to initialize
the bias to a fixed value. We fix $b = \frac{1}{2}$ for simplicity,
but a similar result can be given for any choice of
$b \in \left[0,\frac{1}{2}\right]$.}.
Denote the hinge-loss of the network on the population by:
\[
\mathcal{L}(\mathcal{N}_{\tW, \mB}) =
\mean{(x,y) \sim \mathcal{D}^n}
{\max \{1-y\mathcal{N}_{\tW, \mB}(x),0\}}
\]
Then with probability at least $1-\delta$ we have:
\begin{enumerate}
\item $
\norm{\frac{\partial}{\partial \tW} \mathcal{L}(\mathcal{N}_{\tW, \mB})}_{\max}
\le 5\left(P(n\p) - \frac{1}{2}\right)$ \\
$\norm{\frac{\partial}{\partial \mB} \mathcal{L}(\mathcal{N}_{\tW, \mB})}_{\max} \le
3\left(P(n\p) - \frac{1}{2}\right)$
\item
$L_{\mathcal{D}_n}(\net) \ge
\left(\frac{3}{2} - P(n\p)\right)\left(1-P(n\p)\right)$
\end{enumerate}
Where we denote $\norm{\tA}_{\max} = \max |a_{i_1, \dots, i_j}|$
for some tensor $\tA$.
\end{theorem}

We give the full proof of the theorem in the appendix, and show a sketch
of the argument here. Observe the distribution
$\mathcal{D}_n$, limited to the set $C_{n\p}$.
Notice that the volume of $C_{n\p}$ (namely, the sum
of the lengths of its intervals) decreases exponentially
fast with $n\p$. Since each neuron of the network corresponds
to a separation of the space by a hyper-plane,
we get that the probability of each hyper-plane to separate
an interval of $C_{n\p}$ decreases exponentially fast with
$n\p$. Thus, for $n\p$ that is logarithmic in the number of
neurons, there is a high probability that each interval
of $C_{n\p}$ is not separated by any neuron. 
In this case the network is linear on each interval.
A simple argument gives a bound on the
gradient of a linear classifier on each interval, and on its classification error. Note that the approximation curve determines how much of the distribution
is concentrated on the set $C_{n\p}$. That is, if $P(n\p) - \frac{1}{2}$ is close to zero,
this means that most of the distribution is concentrated on $C_{n\p}$.
Using this property allows us to bound the norm of the gradient in terms of the
approximation curve.

We now give some important implications of this theorem.
First, notice that we can define cantor distributions for which a gradient-based
algorithm fails with high probability.
Indeed, we define the ``fine'' cantor distribution to be a distribution
concentrated on the highest level of the cantor set. Given our
previous definition, this means $p_1, \dots, p_{n-1} = 0$ and $p_n = 1$.
The approximation curve for this distribution is therefore $P(0) = \dots = P(n-1) = \frac{1}{2}$, $P(n) = 1$.
Figure \ref{fig:cantor_intervals} shows the ``fine''
cantor distribution drawn over its composing intervals.
From \thmref{thm:hardness} we get that for
$n > \log^{-1}(\frac{3}{2})\log(\frac{4tk^2}{\delta})$,
with probability at least $1-\delta$,
the population gradient is zero and the error is $\frac{1}{2}$.
This result immediately implies that vanilla gradient-descent
on the distribution will be stuck in the first step.
But SGD, or GD on a finite sample, may move from the initial point,
due to the stochasticity of the gradient estimation.
What the theorem shows is that the objective is extremely
flat almost everywhere in the regime of $\tW$, so stochastic
gradient steps are highly unlikely to converge to any solution
with error better than $\frac{1}{2}$.

\begin{figure}[t]
\begin{center}
\def\w{6}
\colorlet{lightblue}{blue!80}
\begin{tikzpicture}[decoration=Cantor set,line width=3mm]
  \draw [draw=orange] (0,0) -- (\w,0);
  \draw [draw=lightblue] decorate{ (0,0) -- (\w,0) }
  		node[right] {$C_1$};
  \draw [draw=orange] decorate{ (0,-0.5) -- (\w,-0.5) };
  \draw [draw=lightblue] decorate{ decorate{ (0,-0.5) -- (\w,-0.5) }}
  		node[right] {$C_2$};
  \draw [draw=orange] decorate{ decorate{ (0,-1) -- (\w,-1) }};
  \draw [draw=lightblue] decorate {decorate{ decorate{ (0,-1) -- (\w,-1) }}}
  		node[right] {$C_3$};
  \draw [draw=orange] decorate{ decorate{ decorate{ (0,-1.5) -- (\w,-1.5) }}};
  \draw [draw=lightblue] decorate{ decorate{ decorate{
  		decorate{ (0,-1.5) -- (\w,-1.5) }}}}	node[right] {$\vdots$};
\end{tikzpicture}
\caption{``Fine'' cantor distributions of growing depth. Negative areas in orange, positive  in blue.}
\label{fig:cantor_intervals}
\end{center}
\vskip -0.2in
\end{figure}
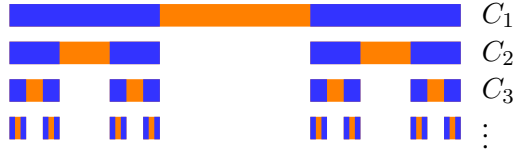

The above argument shows that there are 
fractal distributions that can be \textbf{realized}
by deep networks, for which a standard optimization
process is likely to fail. We note that this result is interesting by itself,
in the broader context of depth separation results.
It implies that for many deep architectures,
there are distributions with depth separation property that cannot be learned by standard optimization algorithms:

\begin{corollary}
\label{crl:hardness_existance}
There exist two constants $c_1, c_2$, such that
for every width $k \ge 10$ and $\delta \in (0,1)$, for
every depth $t > c_1 \log(\frac{k}{\delta}) + c_2$ there exists
a distribution $\mathcal{D}$ on $\reals \times \{\pm 1\}$ for which:
\begin{enumerate}
\item $\mathcal{D}$ can be realized by a neural network of depth
$t$ and width $10$.
\item $\mathcal{D}$ cannot be realized by a one-hidden layer network
with less than $2^{t-1}$ units.
\item Any gradient-based algorithm trying to learn a neural-network
of depth $t$ and width $k$,
with initialization and loss described in \thmref{thm:hardness},
returns a network with error $\frac{1}{2}$
w.p $\ge 1-\delta$.
\end{enumerate}
\end{corollary}

We can go further, and use \thmref{thm:hardness}
to give a better characterization of these hard distributions.
Recall that in the previous section
we showed distributions that exhibit a \emph{strong} depth separation
property: distributions that are realizable by deep networks, for which shallow
networks get an error exponentially close to $\frac{1}{2}$.
From \thmref{thm:hardness} we get that any cantor distribution that gives
a strong depth separation \textbf{cannot} be learned by gradient-based
algorithms:

\begin{corollary}
Fix some depth $t$, width $k$ and some $\delta \in (0,1)$.
Let $n > 4\log^{-1}(\frac{3}{2})\log(\frac{4tk^2}{\delta})+2$.
Let $\mathcal{D}_n$ be some cantor distribution such that any network of width $10$
and depth $t\p < n$ has an error of at least $\frac{1}{2} - \epsilon^{n-t\p}$,
for some $\epsilon \in (0,1)$ (i.e, strong depth separation). Assume we initialize
a network of depth $t$ and width $k$ as described in \thmref{thm:hardness}.
Then with probability at least $1- \delta$:
\begin{enumerate}
\item $
\norm{\frac{\partial}{\partial \tW} \mathcal{L}(\mathcal{N}_{\tW, \mB})}_{\max}
\le 5 \epsilon^{n/2}$ \\
$\norm{\frac{\partial}{\partial \mB} \mathcal{L}(\mathcal{N}_{\tW, \mB})}_{\max} \le
3 \epsilon^{n/2}$
\item
$L_{\mathcal{D}_n}(\net) \ge
\frac{1}{2} - \frac{3}{2} \epsilon^{n/2}$
\end{enumerate}
\end{corollary}

\begin{proof}
Using \thmref{thm:approximation_upper_bound} and the strong depth separation 
property we get that for every $t\p$ we have:
$1-P\left(\frac{t\p-1}{2}\right) \ge L_{\mathcal{D}}(\mathcal{H}_{10,t\p})
\ge \frac{1}{2} - \epsilon^{n-t\p}$.
Choosing $t\p = \frac{n}{2}$ and taking $n\p = \frac{n}{4}-\frac{1}{2}$
we get $P(n\p) \le \frac{1}{2} + \epsilon^{n/2}$.
By the choice of $n$ we can apply \thmref{thm:hardness} and get the required.
\end{proof}

This shows that in the strong depth separation case,
the population gradient is exponentially close to zero with high probability.
Effectively, this property means that even a small amount of stochastic noise in the gradient estimation (for example, in SGD), makes the algorithm fail.

This result gives a very important property of cantor distributions.
It shows that every cantor distribution that cannot be \textbf{approximated}
by a \textbf{shallow} network
(achieving error greater than $\frac{1}{2}$), cannot be \textbf{learned} by a
\textbf{deep} network
(when training with gradient-based algorithms).
While we show this in a very restricted case, we conjecture that this property
holds in general:

\begin{conjecture}
Let $\mathcal{D}$ be some distribution such that $L_\mathcal{D}(\mathcal{H}_{k,t}) = 0$
(realizable with networks of width $k$ and depth $t$).
If $L_{\mathcal{D}}(\mathcal{H}_{k,t\p})$
is exponentially close to $\frac{1}{2}$ when $t\p \to 1$,
then any gradient-based algorithm
training a network of depth $t$ and width $k$ will fail with high probability.
\end{conjecture}

\newlength\figureheight
\newlength\figurewidth
\begin{figure}[t]
\begin{center}
\centerline{
\setlength\figureheight{6cm}
\setlength\figurewidth{7.5cm}
\begin{tikzpicture}

\definecolor{color1}{rgb}{0.203921568627451,0.541176470588235,0.741176470588235}
\definecolor{color0}{rgb}{0.886274509803922,0.290196078431373,0.2}
\definecolor{color3}{rgb}{0.984313725490196,0.756862745098039,0.368627450980392}
\definecolor{color2}{rgb}{0.596078431372549,0.556862745098039,0.835294117647059}

\begin{axis}[
axis background/.style={fill=white!89.80392156862746!black},
axis line style={white},
height=\figureheight,
legend cell align={left},
legend style={at={(0.97,0.03)}, anchor=south east, draw=white!80.0!black, fill=white!89.80392156862746!black},
tick align=outside,
tick pos=left,
title={cantor5},
width=\figurewidth,
x grid style={white},
xlabel={width},
xmajorgrids,
xmin=-9.5, xmax=419.5,
y grid style={white},
ylabel={accuracy},
ymajorgrids,
ymin=0.59361, ymax=1.01139
]
\addlegendimage{no markers, color0}
\addlegendimage{no markers, color1}
\addlegendimage{no markers, color2}
\addlegendimage{no markers, lightgray!62.22222222222222!black}
\addlegendimage{no markers, color3}
\addplot [semithick, color0, mark=*, mark size=3, mark options={solid}]
table [row sep=\\]{%
10	0.6126 \\
20	0.6928 \\
50	0.7246 \\
100	0.7266 \\
200	0.724 \\
400	0.7392 \\
};
\addplot [semithick, color1, mark=*, mark size=3, mark options={solid}]
table [row sep=\\]{%
10	0.7222 \\
20	0.8102 \\
50	0.824 \\
100	0.8574 \\
200	0.8592 \\
400	0.8716 \\
};
\addplot [semithick, color2, mark=*, mark size=3, mark options={solid}]
table [row sep=\\]{%
10	0.778 \\
20	0.8356 \\
50	0.8776 \\
100	0.9236 \\
200	0.9308 \\
400	0.9318 \\
};
\addplot [semithick, lightgray!62.22222222222222!black, mark=*, mark size=3, mark options={solid}]
table [row sep=\\]{%
10	0.823 \\
20	0.865 \\
50	0.9148 \\
100	0.9544 \\
200	0.9718 \\
400	0.9716 \\
};
\addplot [semithick, color3, mark=*, mark size=3, mark options={solid}]
table [row sep=\\]{%
10	0.8064 \\
20	0.8658 \\
50	0.9504 \\
100	0.9688 \\
200	0.9924 \\
400	0.9828 \\
};
\end{axis}

\end{tikzpicture}
\begin{tikzpicture}

\definecolor{color1}{rgb}{0.203921568627451,0.541176470588235,0.741176470588235}
\definecolor{color0}{rgb}{0.886274509803922,0.290196078431373,0.2}
\definecolor{color3}{rgb}{0.984313725490196,0.756862745098039,0.368627450980392}
\definecolor{color2}{rgb}{0.596078431372549,0.556862745098039,0.835294117647059}

\begin{axis}[
axis background/.style={fill=white!89.80392156862746!black},
axis line style={white},
height=\figureheight,
legend cell align={left},
legend entries={{t=1},{t=2},{t=3},{t=4},{t=5}},
legend style={at={(0.97,0.03)}, anchor=south east, draw=white!80.0!black, fill=white!89.80392156862746!black},
tick align=outside,
tick pos=left,
title={cantor5},
width=\figurewidth,
x grid style={white},
xlabel={\#parameters},
xmajorgrids,
xmin=-32117, xmax=675359,
y grid style={white},
ylabel={accuracy},
ymajorgrids,
ymin=0.59361, ymax=1.01139
]
\addlegendimage{mark=*, color0}
\addlegendimage{mark=*, color1}
\addlegendimage{mark=*, color2}
\addlegendimage{mark=*, lightgray!62.22222222222222!black}
\addlegendimage{mark=*, color3}
\addplot [semithick, color0, mark=*, mark size=3, mark options={solid}]
table [row sep=\\]{%
41	0.6126 \\
81	0.6928 \\
201	0.7246 \\
401	0.7266 \\
801	0.724 \\
1601	0.7392 \\
};
\addplot [semithick, color1, mark=*, mark size=3, mark options={solid}]
table [row sep=\\]{%
151	0.7222 \\
501	0.8102 \\
2751	0.824 \\
10501	0.8574 \\
41001	0.8592 \\
162001	0.8716 \\
};
\addplot [semithick, color2, mark=*, mark size=3, mark options={solid}]
table [row sep=\\]{%
261	0.778 \\
921	0.8356 \\
5301	0.8776 \\
20601	0.9236 \\
81201	0.9308 \\
322401	0.9318 \\
};
\addplot [semithick, lightgray!62.22222222222222!black, mark=*, mark size=3, mark options={solid}]
table [row sep=\\]{%
371	0.823 \\
1341	0.865 \\
7851	0.9148 \\
30701	0.9544 \\
121401	0.9718 \\
482801	0.9716 \\
};
\addplot [semithick, color3, mark=*, mark size=3, mark options={solid}]
table [row sep=\\]{%
481	0.8064 \\
1761	0.8658 \\
10401	0.9504 \\
40801	0.9688 \\
161601	0.9924 \\
643201	0.9828 \\
};
\end{axis}

\end{tikzpicture}}
\caption{The effect of depth on learning the cantor set.} \label{fig:experiments_depth}
\end{center}
\vskip -0.4in
\end{figure}
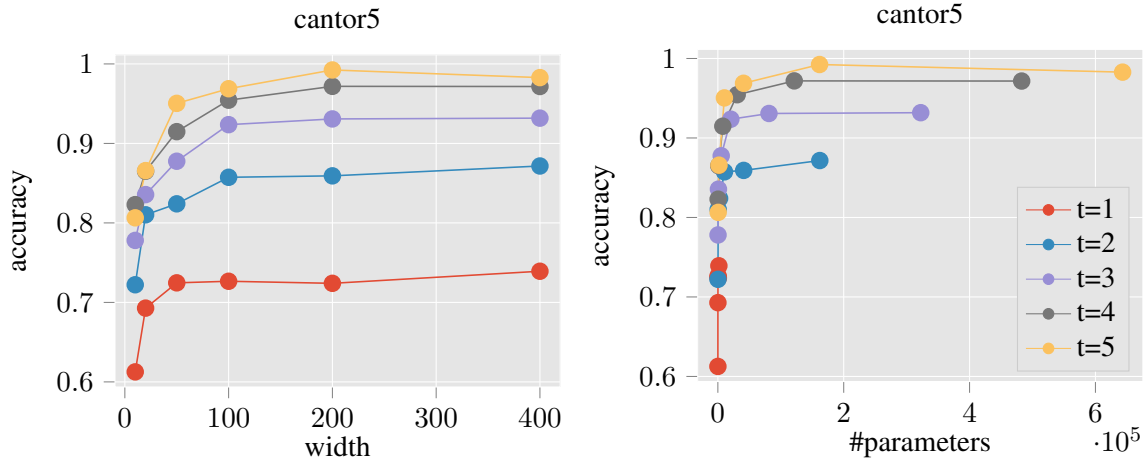

\section{Experiments}
\label{sec:experiments}
In the previous section, we saw that learning a ``fine'' distribution with
gradient-based algorithms is likely to fail. To complete the picture,
we now show a positive result, asserting that when the distribution has
enough weight on the ``coarse'' details,
SGD succeeds to learn a deep network with small error.
Moreover, we show that when training on such distributions,
a clear depth separation is observed, and deeper networks indeed perform better
than shallow networks.
Unfortunately, giving theoretical
evidence to support this claim seems out of reach,
as analyzing gradient-descent on deep networks
proves to be extremely hard due to the dynamics of the non-convex optimization.
Instead, we perform experiments to show these desired properties.

In this section we present our experimental results on learning deep networks
with Adam optimizer (\cite{kingma2014adam}), trained on samples from fractal distributions.
First, we show that depth separation is observed when training on fractal
distributions with ``coarse'' approximation curve:
deeper networks perform better and have better
parameter utilization. Second, we demonstrate the effect of training
on ``coarse'' vs. ``fine'' distributions, showing that the performance
of the network degrades as the approximation curve becomes finer.

We start by observing a distribution with a ``coarse'' approximation curve
(denoted curve \#1), where the negative examples
are evenly distributed between the levels.
The underlying fractal structure is a two-dimensional variant
of the cantor set. This set is constructed by an IFS with four
mappings, each one maps the structure to a rectangle in a different
corner of the space. The negative examples are concentrated in the central
rectangle of each structure. The distributions are shown in \figref{fig:cantor_data}.

We train feed-forward networks of varying depth and width
on a 2D cantor distribution of depth 5. We sample 50K examples
for a train dataset and 5K examples for a test dataset.
We train the networks on this dataset with Adam optimizer for
$10^6$ iterations, with batch size of $100$ and different learning rates.
We observe the best performance
of each configuration (depth and width) on the test data along the runs.
The results of these experiments are shown in \figref{fig:experiments_depth}.

In this experiment, we see that a wide enough depth 5 network
gets almost zero error. The fact that the network needs much
more parameters than the best possible network is not surprising,
as previous results have shown that over-parametrization is essential
for the optimization to succeed (\cite{brutzkus2017sgd}).
Importantly, we can see a clear depth separation between the networks:
deeper networks achieve better accuracy,
and are more efficient in utilizing the network parameters.

Next, we observe the effect of the approximation curve
on learning the distribution. We compare the performance of the best
depth $5$ networks,
when trained on distributions with different
approximation curves. The training and validation process is as described
previously.
We also plot the value of the approximation curve for each distribution,
in levels $3,4,5$ of the fractal.
The results of this experiment are shown in \figref{fig:experiments_curve}.
Clearly, the approximation curve
has a crucial effect on learning the distribution.
While for ``coarse'' approximation curves the network achieves an error
that is close to zero,
we can see that distributions with ``fine'' approximation
curves result in a drastic degradation in performance.

\begin{figure}[t]
\begin{center}
\centerline{
\setlength\figureheight{7cm}
\setlength\figurewidth{8cm}
\begin{tikzpicture}

\definecolor{color1}{rgb}{0.203921568627451,0.541176470588235,0.741176470588235}
\definecolor{color0}{rgb}{0.886274509803922,0.290196078431373,0.2}
\definecolor{color2}{rgb}{0.596078431372549,0.556862745098039,0.835294117647059}

\begin{axis}[
axis background/.style={fill=white!89.80392156862746!black},
axis line style={white},
height=\figureheight,
legend cell align={left},
legend entries={{acc},{P(5)},{P(4)},{P(3)}},
legend style={at={(0.03,0.03)}, anchor=south west, draw=white!80.0!black, fill=white!89.80392156862746!black},
tick align=outside,
tick pos=left,
width=\figurewidth,
x grid style={white},
xlabel={\#curve},
xmajorgrids,
xmin=0.75, xmax=6.25,
y grid style={white},
ylabel={accuracy},
ymajorgrids,
ymin=0.475, ymax=1.025,
ytick={0.4,0.5,0.6,0.7,0.8,0.9,1,1.1},
yticklabels={0.4,0.5,0.6,0.7,0.8,0.9,1.0,1.1}
]
\addlegendimage{mark=*, color0}
\addlegendimage{no markers, color1, dashed}
\addlegendimage{no markers, color2, dashed}
\addlegendimage{no markers, lightgray!62.22222222222222!black, dashed}
\addplot [semithick, color0, mark=*, mark size=3, mark options={solid}]
table [row sep=\\]{%
1	0.9924 \\
2	0.9654 \\
3	0.9096 \\
4	0.9322 \\
5	0.7566 \\
6	0.7764 \\
};
\addplot [semithick, color1, dashed]
table [row sep=\\]{%
1	1 \\
2	1 \\
3	1 \\
4	1 \\
5	1 \\
6	1 \\
};
\addplot [semithick, color2, dashed]
table [row sep=\\]{%
1	0.925650822501483 \\
2	0.870039475052563 \\
3	0.792893218813452 \\
4	0.684241121587126 \\
5	0.583822107898522 \\
6	0.5 \\
};
\addplot [semithick, lightgray!62.22222222222222!black, dashed]
table [row sep=\\]{%
1	0.840246044613553 \\
2	0.7062994740159 \\
3	0.5 \\
4	0.542312339311059 \\
5	0.506242576182648 \\
6	0.5 \\
};
\end{axis}

\end{tikzpicture}
\setlength\figureheight{2.8cm}
\setlength\figurewidth{3.3cm}
\begin{tikzpicture}

\definecolor{color0}{rgb}{0.886274509803922,0.290196078431373,0.2}

\begin{groupplot}[group style={group size=2 by 3}]
\nextgroupplot[
axis background/.style={fill=white!89.80392156862746!black},
axis line style={white},
height=\figureheight,
tick align=outside,
tick pos=left,
title={curve \#1},
width=\figurewidth,
x grid style={white},
xmajorgrids,
xmin=-0.25, xmax=5.25,
y grid style={white},
ymajorgrids,
ymin=0.475, ymax=1.025,
xticklabels={,,}
]
\addplot [semithick, color0, forget plot]
table [row sep=\\]{%
0	0.5 \\
1	0.629449436703876 \\
2	0.742141716744801 \\
3	0.840246044613553 \\
4	0.925650822501483 \\
5	1 \\
};
\nextgroupplot[
axis background/.style={fill=white!89.80392156862746!black},
axis line style={white},
height=\figureheight,
tick align=outside,
tick pos=left,
title={curve \#2},
width=\figurewidth,
x grid style={white},
xmajorgrids,
xmin=-0.25, xmax=5.25,
y grid style={white},
ymajorgrids,
ymin=0.475, ymax=1.025,
xticklabels={,,},
yticklabels={,,}
]
\addplot [semithick, color0, forget plot]
table [row sep=\\]{%
0	0.5 \\
1	0.5 \\
2	0.5 \\
3	0.7062994740159 \\
4	0.870039475052563 \\
5	1 \\
};
\nextgroupplot[
axis background/.style={fill=white!89.80392156862746!black},
axis line style={white},
height=\figureheight,
tick align=outside,
tick pos=left,
title={curve \#3},
width=\figurewidth,
x grid style={white},
xmajorgrids,
xmin=-0.25, xmax=5.25,
y grid style={white},
ymajorgrids,
ymin=0.475, ymax=1.025,
xticklabels={,,}
]
\addplot [semithick, color0, forget plot]
table [row sep=\\]{%
0	0.5 \\
1	0.5 \\
2	0.5 \\
3	0.5 \\
4	0.792893218813452 \\
5	1 \\
};
\nextgroupplot[
axis background/.style={fill=white!89.80392156862746!black},
axis line style={white},
height=\figureheight,
tick align=outside,
tick pos=left,
title={curve \#4},
width=\figurewidth,
x grid style={white},
xmajorgrids,
xmin=-0.25, xmax=5.25,
y grid style={white},
ymajorgrids,
ymin=0.475, ymax=1.025,
xticklabels={,,},
yticklabels={,,}
]
\addplot [semithick, color0, forget plot]
table [row sep=\\]{%
0	0.5 \\
1	0.500156631162693 \\
2	0.505155895716397 \\
3	0.542312339311059 \\
4	0.684241121587126 \\
5	1 \\
};
\nextgroupplot[
axis background/.style={fill=white!89.80392156862746!black},
axis line style={white},
height=\figureheight,
tick align=outside,
tick pos=left,
title={curve \#5},
width=\figurewidth,
x grid style={white},
xmajorgrids,
xmin=-0.25, xmax=5.25,
y grid style={white},
ymajorgrids,
ymin=0.475, ymax=1.025
]
\addplot [semithick, color0, forget plot]
table [row sep=\\]{%
0	0.5 \\
1	0.500000310068408 \\
2	0.500159052467644 \\
3	0.506242576182648 \\
4	0.583822107898522 \\
5	1 \\
};
\nextgroupplot[
axis background/.style={fill=white!89.80392156862746!black},
axis line style={white},
height=\figureheight,
tick align=outside,
tick pos=left,
title={curve \#6},
width=\figurewidth,
x grid style={white},
xmajorgrids,
xmin=-0.25, xmax=5.25,
y grid style={white},
ymajorgrids,
ymin=0.475, ymax=1.025,
yticklabels={,,}
]
\addplot [semithick, color0, forget plot]
table [row sep=\\]{%
0	0.5 \\
1	0.5 \\
2	0.5 \\
3	0.5 \\
4	0.5 \\
5	1 \\
};
\end{groupplot}

\end{tikzpicture}}
\caption{Learning depth 5 network on 2D cantor set of depth 5, with different approximation curves.
The figures show the values of the approximation curve
(denoted $P$) at different levels of the fractal.
Large values correspond to more weight.
In red is the accuracy of the best depth 5 network architecture trained on these distributions.} \label{fig:experiments_curve}
\end{center}
\vskip -0.2in
\end{figure}
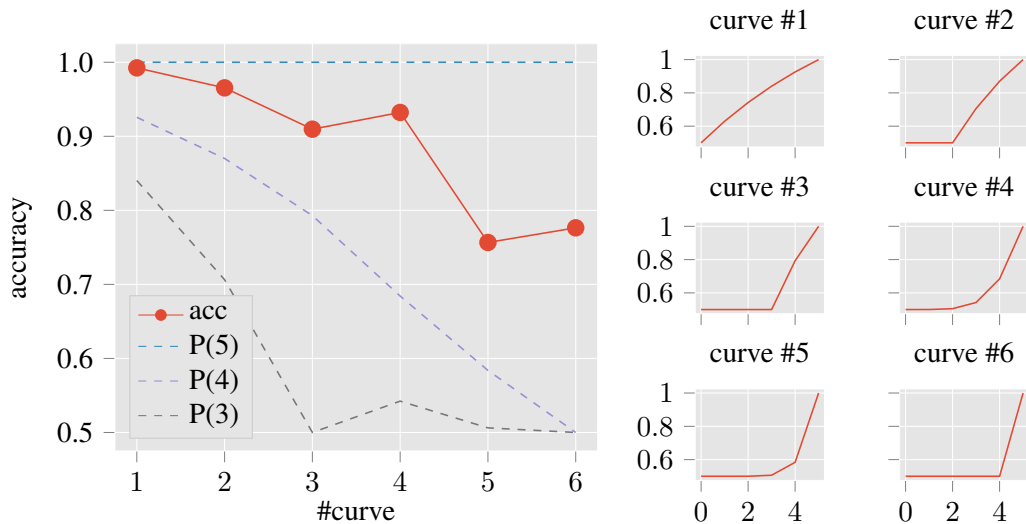

We perform the same experiments with different fractal structures
(\figref{fig:ifs} at the beginning of the paper shows
an illustration of the distributions we use).
Tables \ref{experiments_all_fractals}, \ref{experiments_all_fractals_cruves} in the appendix summarize the results of all the experiments.
We note that the effect of depth can be seen clearly in all fractal structures. The effect of the approximation curve is observed in the
Cantor set and the Pentaflake and Vicsek sets (fractals generated by an IFS with 5 transformations). In the Sierpinsky Triangle (generated with 3 transformations), the approximation curve
seems to have no effect when the width of the network is large enough.
This might be due to the fact that a depth 5 IFS with 3 transformations
generates a relatively small number of linear regions, making
the problem overall relatively easy, even when the underlying distribution
is hard.

\paragraph{Acknowledgements:} This research is supported by the European Research Council (TheoryDL project).

\clearpage

\bibliography{fractals}
\bibliographystyle{plain}

\clearpage
\onecolumn
\appendix

\section{Proof of \thmref{thm:depth_expressivity}}
To prove the theorem, we begin with two technical lemmas:
\begin{lemma}
\label{lem:box}
For every $\epsilon > 0$, there exists a neural-network of width $3d$ with two hidden-layers ($k=3d,t=3$)
such that
$\mathcal{N}_{\tW, \mB}(\vx) = \vx$ for $\vx \in [0,1]^d$,
and $\mathcal{N}_{\tW, \mB}(\vx) = 0$ for $\vx \in \reals^d$
with $d(\vx, [0,1]^d) = \min_{\vy \in [0,1]^d} \norm{\vx -\vy} > \epsilon$.
\end{lemma}
\begin{proof}
Let $N > 0$ be some constant, and observe the function:
\[
f_i(\vx) = \sigma(\sigma(x_i) - N \sum_{j=1}^d \sigma(-x_j) -N\sum_{j=1}^d \sigma(x_j-1))
\]
Notice that $f_i(\vx) = x_i$ for $\vx \in [0,1]^d$,
and that $f_i(\vx) = 0$ if $d(\vx, [0,1]^d) > \epsilon$, when taking $N$
to be large enough. Since $f(\vx) = (f_1(\vx), \dots, f_d(\vx))$ is
a two hidden layer neural-network of width $3d$, 
the required follows.
\end{proof}

\begin{lemma}
\label{lem:box_indicator}
For every $\gamma > 0$, there exists a neural-network of width $2d$ with two hidden-layers ($k=2d,t=3$) such that
$\mathcal{N}_{\tW, \mB}(\vx) = 1$ for $\vx \notin [0,1]^d$,
and $\mathcal{N}_{\tW, \mB}(\vx) = 0$ for $\vx \in [\gamma,1-\gamma]^d$.
\end{lemma}
\begin{proof}
Let $N > 0$ be some constant, and observe the function:
\[
\tilde{f}(\vx) = 1-\sigma(1-N \sum_{j=1}^d \sigma(\gamma-x_j) -N\sum_{j=1}^d \sigma(x_j-1+\gamma))
\]
Notice that $\tilde{f}(\vx) = 0$ for $\vx \in [\gamma,1-\gamma]^d$,
and that $\tilde{f}(\vx) = 0$
if $\vx \notin [0,1]^d$, when taking $N$
to be large enough. Since $\tilde{f}$ a two hidden layer neural-network of width $2d$, 
the required follows.
\end{proof}

The next lemmas will show how a single block of the network operates
on the set $K_n$:

\begin{lemma}
\label{lem:one_level}
There exists a neural-network of width $\max\{dr,3d\}$ with two hidden-layers
($k=3dr, t=3$)
such that for any $n$ we have:
\begin{enumerate}
\item $\mathcal{N}_{\tW, \mB}(K_n) \subseteq K_{n-1}$
\item $\mathcal{N}_{\tW, \mB}(K_1 \setminus K_n)  \subseteq
\mathcal{X} \setminus K_{n-1}$
\end{enumerate}
\end{lemma}
\begin{proof}
As an immediate corollary from \lemref{lem:box}, there exists
$f : \reals^d \rightarrow \reals^d$, that can be implemented by a neural network
with two hidden-layers and width $3d$, such that $f(\vx) = \vx$
for $\vx \in K_0$ and $f(\vx) = 0$
if $d(\vx, K_0) > \frac{\epsilon}{2}$.
Define the following function:
\[
g(\vx) = \sum_{i=1}^r f \left( (\mM^{(i)})^{-1}\vx-(\mM^{(i)})^{-1}\vv^{(i)} \right) =
\sum_{i=1}^r f \left( F_i^{-1}(\vx) \right)
\]

Notice that for every $\vx \in \mathcal{X}$ there is at most one $i \in [r]$
such that $f(F_i^{-1}(\vx)) > 0$. Indeed, assume there are $i\ne j \in [r]$
such that $f(F_i^{-1}(\vx)) > 0$ and $f(F_j^{-1}(\vx)) > 0$.
Therefore, $d(F_i^{-1}(\vx), K_0) \le \frac{\epsilon}{2}$ and
$d(F_j^{-1}(\vx), K_0) \le \frac{\epsilon}{2}$.
Therefore, there exist $\vy,\vz \in K_0$ such that
$\norm{F_i^{-1}(\vx)-\vy} \le \frac{\epsilon}{2}$ and
$\norm{F_j^{-1}(\vx)-\vz} \le \frac{\epsilon}{2}$.
From this we get:
\begin{align*}
\norm{\vx - F_i(\vy)} &= \norm{\vx-\mM^{(i)}\vy -\vv^{(i)}} \\
&= \norm{\mM^{(i)} \left(
(\mM^{(i)})^{-1} \vx - (\mM^{(i)})^{(-1)} \vv^{(i)} -\vy \right)} \\
&\le \norm{(\mM^{(i)})^{-1} \vx - (\mM^{(i)})^{(-1)} \vv^{(i)} -\vy} \\
&= \norm{F_i^{-1}(\vx) - \vy} \le \frac{\epsilon}{2}
\end{align*}
where we use the fact that $\mM^{(i)}$ is a contraction.
Similarly, we get that $\norm{\vx - F_j(\vz)} \le \frac{\epsilon}{2}$,
so this gives us $\norm{F_i(\vy) - F_j(\vz)} \le \epsilon$.
Since $\vy, \vz \in K_0$, this is contradiction to Assumption
\ref{asm:disjoint}.

We now show the following:
\begin{enumerate}
\item $g(K_n) \subseteq K_{n-1}$:

Let $\vx \in K_n$,
and denote $i \in [r]$ the unique $i$ for which
$\vx \in F_i(K_{n-1}) \subseteq F_i(K_0)$.
From the properties of $f$, we get that $f(F_i^{-1}(\vx)) = F_i^{-1}(\vx)$
and $f(F_j^{-1}(\vx)) = 0$ for $j \ne i$,
so $g(\vx) = f(F_i^{-1}(\vx)) = F_i^{-1} (\vx) \in K_{n-1}$.

\item $g(K_1 \setminus K_n) \subseteq
\mathcal{X} \setminus K_{n-1}$:

Let $\vx \in K_1 \setminus K_n$
and assume by contradiction that $g(\vx) \in K_{n-1}$.
Let $i \in [r]$ be the unique $i$ such that
$\vx \in F_i(K_0)$ and we have seen that in this case
$g(\vx) = F_i^{-1}(\vx)$, so $F_i^{-1}(\vx) \in K_{n-1}$
and therefore $\vx \in F_i(K_{n-1}) \subseteq K_n$ in contradiction
to the assumption.
\end{enumerate}

Since $g$ can be implemented with a neural network of width $3dr$
and two hidden-layer, this completes the proof of the lemma.
\end{proof}

\begin{lemma}
\label{lem:one_level_comp}
There exists a neural-network of width $2dr$ with two hidden-layers
($k=2dr, t=3$) such that for any $n$ we have:
\begin{enumerate}
\item $\mathcal{N}_{\tW, \mB}(\mathcal{X} \setminus K_1) = \{1\}$
\item $\mathcal{N}_{\tW, \mB}(K_1^\gamma) = \{0\}$
\end{enumerate}
\end{lemma}

\begin{proof}
As a corollary of \lemref{lem:box_indicator}, there exists
$\tilde{f}: \reals^d \to \reals$,
a two hidden-layer neural-network of width $2d$,
such that $\tilde{f}(\vx) = 1$ for $\vx \notin K_0$ and
$\tilde{f}(\vx) = 0$ for $\vx \in K_0^\gamma$.
Now, define:
\[
\tilde{g}(\vx) = 1 -r + \sum_{i=1}^r \tilde{f} \left( F_i^{-1}(\vx) \right)
\]
We show the following:
\begin{enumerate}
\item $\tilde{g}(\mathcal{X} \setminus K_1) = \{1\}$:

Let $\vx \notin K_1 = \cup_i F_i(K_0)$, then for every $i$
we have $\vx \notin F_i(K_0)$ and hence $F_i^{-1}(\vx) \notin K_0$
so $\tilde{f}(F_i^{-1}(\vx)) = 1$ and so $\tilde{g}(\vx) = 1$.

\item $\tilde{g}(K_1^\gamma) = \{0\}$:

Let $\vx \in K_1^\gamma$, and let $i$ be the unique index such that
$\vx \in F_i(K_0)$. So we have $\tilde{f}(F_i^{-1}(\vx)) = 0$
and for all $j \ne i$ we have $\tilde{f}(F_j^{-1}(\vx)) = 1$,
and therefore $\tilde{g}(\vx) = 0$.
\end{enumerate}
And $\tilde{g}$ can be implemented by a width $2dr$ two hidden-layer
network.
\end{proof}

\begin{proof} of \thmref{thm:depth_expressivity}
Let $g, \tilde{g}$ as defined in the previous lemmas.
Denote $h_0 : \reals^d \to \reals^{d+1}$ the function:
\[
h_0(\vx) = [g(\vx), \tilde{g}(\vx)]
\]
and denote $h : \reals^{d+1} \to \reals^{d+1}$ the function:
\[
h(\vx) = [g(\vx_{1 \dots d}), x_{d+1}+\tilde{g}(\vx_{1 \dots d})]
\]
Denote $h^n$ the composition of $h$ on itself $n$ times,
and observe the network defined by $H = h^{n-1} \circ h_0$.
Note that $H$ satisfies the following properties:
\begin{enumerate}
\item For $\vx \in K_n^\gamma$ we have $H(\vx)_{d+1} = 0$:
indeed, by iteratively applying the previous lemmas,
we get that $g^j(\vx) \in K_1^\gamma$ for every $j \le n-1$,
and therefore $\tilde{g}(g^j(\vx)) = 0$ for every $j \le n-1$.
Observe that: $H(\vx)_{d+1} = \sum_{j=1}^{n-1} \tilde{g}(g^j(\vx)) = 0$.

\item For $\vx \notin K_n$ we have $H(\vx)_{d+1} \ge 1$:
there exists $K_j$ such that $\vx \in K_j \setminus K_{j+1}$,
so by applying \ref{lem:one_level} we get $g^j(\vx) \notin K_1$,
so $\tilde{g}(g^j(\vx)) = 1$, and therefore $H(\vx)_{d+1} \ge 1$
(since the summation is over positive values).
\end{enumerate}

Therefore, composing $H(\vx)$ with a linear threshold on $H(\vx)_{d+1}$
gives a network as required. Since ever block of $H$ is has two
hidden-layers of width $5dr$, we get that this network has depth $2n+1$
and width $5dr$.
\end{proof}

\section{Proof of \crlref{crl:depth_expressivity_multi}}
\begin{proof}
As mentioned, we can rewrite the IFS with $r^s$ transformations,
generating $K_{s \cdot \floor{n/s}}$ in $\floor{n/s}$
iterations. Therefore, using the construction of
\thmref{thm:depth_expressivity}, we have a network of depth 
$2 \floor{\frac{n}{s}}$ and width $5dr^s$ that maps $K_{\floor{n/s}}$
to $K_0$, and therefore maps $K_n$ to $K_{n-\floor{n/s}}$.
Now, a two hidden-layer network of width at most $dr^s$
can separate $K_{n-\floor{n/s}}^\gamma$ from
$\mathcal{X} \setminus K_{n-\floor{n/s}}$.
This constructs a network of depth $2 \floor{n/s}+2$
and width $5dr^s$ that achieves the required.
\end{proof}

\section{Proof of \thmref{thm:approximation_lower_bound}}
\begin{proof}
Using again \cite{montufar2017notes}, we get that the number of linear
regions in $\net$ is $r^{st}$. This means that $\net$ crosses zero at most
$r^{st}$ times. Now, fix $n > j > st$, and notice that $K_j$ is a union
of $r^{j}$ intervals, so $K_j = \cup_{i=1}^{r^j} I_i$,
for intervals $I_i$.
By our assumption, we get that for every $i$,
$\prob{(x,y) \sim \mathcal{D}_n}{x \in I_i ~and~ y=-1} = p$ for some $p$,
and from this we get:
\[\prob{(x,y) \sim \mathcal{D}_n}{x \in I_i ~and~ y=-1}
= r^{-j} \prob{(x,y) \sim \mathcal{D}_n}{x \in K_j ~and~ y =-1}\]

We get that there are at most $r^{st}$ intervals
of $K_j$ in which $\net$ crosses zero. Denote $J \subseteq [r^j]$
the subset of intervals on which $\sign(\net)$ is constant, 
and for every $i \in J$ we denote $\hat{y}_i$ such that
$\sign(\net(I_i)) = \{\hat{y}_i\}$.
Notice that:
\begin{align*}
&\prob{(x,y) \sim \mathcal{D}_n}{x \in K_j ~and~ y=1} 
= \frac{1}{2} \\
&\prob{(x,y) \sim \mathcal{D}_n}{x \in K_j ~and~ y=-1} 
= 1 - P(j)
\end{align*}
So the optimal choice for every $\hat{y}_i$ is $1$.
Then we have:
\begin{align*}
\prob{(x,y) \sim \mathcal{D}_n}{\sign(\net(x)) \ne y}
&= \prob{(x,y) \sim \mathcal{D}_n}{\sign(\net(x)) \ne -1
~and~ x \notin K_j}\\
&+ \prob{(x,y) \sim \mathcal{D}_n}{\sign(\net(x)) \ne y
~and~ x \in K_j}\\
&\ge \sum_{i \in [r^j]} \prob{(x,y) \sim \mathcal{D}_n}{\sign(\net(x)) \ne y
~and~ x \in I_i} \\
&\ge \sum_{i \in J} \prob{(x,y) \sim \mathcal{D}_n}{\hat{y}_i \ne y
~and~ x \in I_i} \\
&\ge \sum_{i \in J} \prob{(x,y) \sim \mathcal{D}_n}{y = -1
~and~ x \in I_i} \\
&= |J|r^{-j} \prob{(x,y) \sim \mathcal{D}_n}{x \in K_j ~and~ y=-1} \\
&\ge (1 - r^{st-j})(1-P(j))
\end{align*}
\end{proof}

\section{Proof of \thmref{thm:hardness}}

Observe that for every $n\p$, we can write $C_{n\p}$ as union of $2^{n\p}$ intervals, so
$C_{n\p} = \cup_j I_j$.
We can observe the distribution limited to each of these intervals,
and get the following:

\begin{lemma}
\label{lem:conditional_expectation}
Let $\mathcal{D}_n$ be some cantor distribution (as defined in the paper).
Then:
\begin{align*}
&\abs{\mean{(x,y) \sim \mathcal{D}_n}{y \given x \in I_j}} \le 
2\left( P(n\p)-\frac{1}{2} \right) \\
&\abs{\mean{(x,y) \sim \mathcal{D}_n}{xy \given x \in I_j}} \le 
2\left( P(n\p)-\frac{1}{2} \right) \\
\end{align*}
\end{lemma}

\begin{proof}
Let $I_j$ be some interval of $C_{n\p}$, and let $c_j$ be the central point of $I_j$.
Notice that by definition of the distribution we have:
\begin{align*}
&\prob{(x,y) \sim \mathcal{D}_n}{y = 1 ~and~ x \in I_j} = 2^{-n\p-1} \\
&\prob{(x,y) \sim \mathcal{D}_n}{y = -1 ~and~ x \in I_j} =
2^{-n\p-1}(1 -  \sum_{i=1}^{n\p} p_i) = 2^{-n\p}(1 -P(n\p))
\end{align*}
So we get that $\prob{(x,y) \sim \mathcal{D}_n}{x \in I_j} =
2^{-n\p}(\frac{3}{2} -P(n\p))$, and therefore:
\begin{align*}
\abs{\mean{(x,y) \sim \mathcal{D}_n}{y \given x \in I_j}} &=
\abs{\prob{(x,y) \sim \mathcal{D}_n}{y = 1 \given x \in I_j}
-\prob{(x,y) \sim \mathcal{D}_n}{y = -1 \given x \in I_j}} \\
&= \abs{(P(n\p)-\frac{1}{2})(\frac{3}{2}-P(n\p))^{-1}} \\
&\le 2\left( P(n\p)-\frac{1}{2} \right)
\end{align*}

Notice that from the structure of the set $C_n$, the average of all the points
in $I_j \cap C_n$ is exactly the central point $c_j$ (this is due to the symmetry of
the cantor set around its central point). Similarly, we get that each level
of the negative distribution, $E_i := C_{i-1} \setminus C_i$, its average
is also $c_j$. So we get:
\[
\mean{(x,y) \sim \mathcal{D}_n}{x \given x \in I_j ~and~ y = -1} =
\mean{(x,y) \sim \mathcal{D}_n}{x \given x \in I_j ~and~ y = 1} = c_j
\]
Therefore, we get that:
\begin{align*}
\mean{(x,y) \sim \mathcal{D}_n}{xy \given x \in I_j}
&= c_j\prob{(x,y) \sim \mathcal{D}_n}{y=1 \given x \in I_j} \\
&-c_j\prob{(x,y) \sim \mathcal{D}_n}{y=-1 \given x \in I_j} \\
&= c_j(P(n\p)- \frac{1}{2}) (\frac{3}{2} - P(n\p))^{-1}
\end{align*}
So we have:
\begin{align*}
\abs{\mean{(x,y) \sim \mathcal{D}_n}{xy \given x \in I_j}}
\le 2 \left( P(n\p)- \frac{1}{2} \right)
\end{align*}

\end{proof}

Using this result, we get the following lemma:
\begin{lemma}
\label{lem:zero_grad}
Let $g: \reals \to \reals^k, f: \reals^k \to \reals$ two functions,
and let $\mW \in \reals^{k \times k}, \vc \in \reals^k$,
such that for every $j$,
$g$ is affine on $I_j$ and $f$ is affine on
$\mW g(I_j) + \vc \subseteq \reals^k$.
For every $j$, denote $\vu_j, \vv_j, \va_j, \vb_j \in \reals^k$ such that
for every $\vx \in I_j$:
\[
g(x) = x\vu_j + \va_j, ~
f(\mW g(x) + \vc) = \vv_j^\top (\mW g(x) + \vc) + \vb_j
\]
Assume that $\norm{\vu_j}_\infty, \norm{\vv_j}_\infty,
\norm{\va_j}_\infty,\norm{\vb_j}_\infty \le 1$.
Denote $h : \reals \to \reals$ s.t
$h(x) = f(\mW g(x) + \vc)$.
Then the following holds:
\begin{align*}
&\norm{\mean{(x,y) \sim \mathcal{D}_F}{-y\frac{\partial}{\partial \mW}h(x)
\given x \in C_{n\p}}}_{\max} \le 4 \left( P(n\p)- \frac{1}{2} \right)  \\
&\norm{\mean{(x,y) \sim \mathcal{D}_F}{-y\frac{\partial}{\partial \vc}h(x)
\given x \in C_{n\p}}}_{\infty} \le 2 \left( P(n\p)- \frac{1}{2} \right)  \\
\end{align*}
Where for matrix $\mA$ we denote $\norm{\mA}_{\max} = \max_{i,j} |a_{i,j}|$.
\end{lemma}

\begin{proof}
For every $x \in I_j$ it holds that:
\begin{align*}
&\frac{\partial}{\partial \mW}h(x)
= \frac{\partial}{\partial \mW}
\left[ \vv_j^\top (\mW( \vu_j x + \va_j) + \vc) + \vb_j \right]
= \vu_j \vv_j^\top x + \va_j \vv_j^\top \\
&\frac{\partial}{\partial \vc}h(x)
= \frac{\partial}{\partial \vc}
\left[ \vv_j^\top (\mW( \vu_j x + \va_j) + \vc) + \vb_j \right]
= \vv_j \\
\end{align*}
Using \lemref{lem:conditional_expectation}, we get:
\begin{align*}
\norm{\mean{(x,y) \sim \mathcal{D}_F}
{-y\frac{\partial}{\partial \mW}h(x) \given x \in I_j}}_{\max}
&= \norm{\mean{(x,y) \sim \mathcal{D}_F}
{-y(\vu_j \vv_j^\top x + \va_j \vv_j^\top) \given x \in I_j}}_{\max} \\
&\le \abs{\mean{(x,y) \sim \mathcal{D}_F}
{-y x  \given x \in I_j}} \cdot\norm{\vu_j \vv_j^\top}_{\max} \\
&+ \abs{\mean{(x,y) \sim \mathcal{D}_F}{-y \given x \in I_j}} \cdot \norm{\va_j \vv_j^\top}_{\max} \\
&\le 4 \left( P(n\p)- \frac{1}{2} \right)  \\
\end{align*}
\begin{align*}
\norm{\mean{(x,y) \sim \mathcal{D}_F}
{-y\frac{\partial}{\partial \vc}h(x) \given x \in I_j}}_{\infty}
&= \norm{\mean{(x,y) \sim \mathcal{D}_F}
{-y \vv_j \given x \in I_j}}_{\infty} \\
&= \abs{\mean{(x,y) \sim \mathcal{D}_F}{-y \given x \in I_j}} \cdot
\norm{\vv_j}_{\infty} \le 2 \left( P(n\p)- \frac{1}{2} \right)  \\
\end{align*}
Finally, from this we get:
\begin{align*}
\norm{\mean{(x,y) \sim \mathcal{D}_F}{-y\frac{\partial}{\partial \mW}h(x)
\given x \in C_{n\p}}}_{\max}
&\le \sum_j \prob{}{x \in I_j}
\norm{\mean{(x,y) \sim \mathcal{D}_F}
{-y\frac{\partial}{\partial \mW}h(x) \given x \in I_j}}_{\max} \\
&\le 4 \left( P(n\p)- \frac{1}{2} \right) \\
\end{align*}
\begin{align*}
\norm{\mean{(x,y) \sim \mathcal{D}_F}{-y\frac{\partial}{\partial \vc}h(x)
\given x \in C_{n\p}}}_{\infty}
&\le \sum_j \prob{}{x \in I_j}
\norm{\mean{(x,y) \sim \mathcal{D}_F}
{-y\frac{\partial}{\partial \vc}h(x) \given x \in I_j}}_{\infty} \\
&\le 2 \left( P(n\p)- \frac{1}{2} \right) \\
\end{align*}
\end{proof}

Now, we need to show that with high probability over the initialization
of the network, every layer is affine on $I_j$-s.

\begin{lemma}
\label{lem:affine}
Fix $\delta \in (0,1)$, and let $s \le k$.
Let $g : \reals \to \reals^s$ such that for every $j$, $g$ is affine
and non-expansive on $I_j$ (w.r.t to $\norm{\cdot}_\infty$).
Let $\mW \in \reals^{k \times s}$
a random matrix such that every entry is initialized uniformly from
$[-\frac{1}{2s},\frac{1}{2s}]$,
and let $b > 2k^2 \left(\frac{2}{3}\right)^{n\p} \delta^{-1}$, some fixed bias.
Denote $h(x) := \psi(\mW g(x)+b)$, for some $\psi$ that is affine on
every interval that is bounded away from zero.
Then with probability at least
$1-\delta$, for every $j$, $h(x)$ is affine and non-expansive on $I_j$.
\end{lemma}

\begin{proof}
Denote $\vw_i \in \reals^k$ the $i$-th row of $\mW$.
Fix some $j$, and denote $c_j$ the central point of $I_j$.
We show that:
\[
\prob{\vw_i \sim \mathcal{U}([-\frac{1}{2s},\frac{1}{2s}]^s)}
{|\vw_i^\top g(c_j) + b| \le 3^{-n\p} }
\le \frac{\delta}{k2^{n\p}}
\]
If $\norm{g(c_j)}_\infty \le b$ then 
$|\vw_i^\top g(c_j)| \le \norm{\vw_i}_1 \norm{g(c_j)}_\infty \le \frac{b}{2}$ and therefore 
$|\vw_i^\top g(c_j) + b| \ge \frac{b}{2} > 3^{-n\p}$.
So we can assume $\norm{g(c_j)}_\infty > b$,
and let $\ell \in [k]$ be some index such that
$g(c_j)_\ell > b$.
Now, fix some values for
$w_{i,1}, \dots, w_{i,\ell-1}, w_{i, \ell+1}, \dots, w_{i,k}$,
and observe the distribution of 
$\vw_i^\top g(c_j) + b$ (with respect to the randomness of $w_{i,\ell}$).
Since $w_{i,\ell}$ is uniformly distributed in
$[-\frac{1}{2s},\frac{1}{2s}]$, we get that
this is a uniform distribution over some interval
$J$ with $|J| \ge \frac{b}{s}$.
From this we get:
\begin{align*}
\prob{w_{i,\ell} \sim \mathcal{U}([-\frac{1}{2s},\frac{1}{2s}])}
{|\vw_i^\top g(c_j) + b| \le 3^{-n\p}}
&= \prob{x \sim \mathcal{U}(J)}{x \in [-3^{-n\p}, 3^{-n\p}]} \\
&= |J \cap [-3^{-n\p}, 3^{-n\p}]| / |J| \\
&\le |[-3^{-n\p}, 3^{-n\p}]| / |J| \\
&= \frac{2s3^{-n\p}}{b} \le \frac{\delta}{k2^{n\p}}
\end{align*}
Since there are $2^{n\p}$ intervals $I_j$
and $k$ rows in $\mW$, using the union bound we get that with probability
at least $1-\delta$,
we have for all $j \in [2^{n\p}]$
that $\norm{\mW g(c_j) + b}_\infty > 3^{-n\p}$.
Since we have $|I_j| = 3^{-n\p}$, and since $g$ is non-expansive,
this means that the set $\mW g(I_j) + b$
does not cross zero at any of its coordinates.
Indeed, assume there exists $i \in [k]$ such that $\vw_i^\top g(I_j) + b$
crosses zero, and assume w.l.o.g that $\vw_i^\top g(I_j) + b > 0$. 
Then there exists
$x \in I_j$ with $\vw_i^\top g(x) + b \le 0$ and therefore:
\[
3^{-n\p} < |\vw_i^\top g(x) - \vw_i^\top g(c_j)| 
\le \norm{\vw_i}_1 \norm{g(x)-g(c_j)}_\infty \le |x-c_j| \le
\frac{1}{2}3^{-n\p}
\]
and we reach contradiction.

Since $\psi$
is affine on intervals that are bounded away from zero,
we get that $h(x) = \psi(\mW g(x) + b)$ is affine on all $I_j$.

To show that $h$ is non-expansive on $I_j$, let $x,y \in I_j$,
and from the fact that $g$ is non-expansive we have
$\norm{g(x)-g(y)}_\infty \le |x-y|$. Since we showed that 
$\psi$ is affine on $\mW g(I_j) + b$, we get:
\[
|h(x) - h(y)| = |\psi(\mW g(x)+b) - \psi(\mW g(y)+b)|
= |\psi(\mW (g(x)- g(y)))| \le |\mW (g(x)- g(y))|
\]
Therefore, for every $i$ we get:
\[
|h(x)_i-h(y)_i| = |\vw_i^\top (g(x)-g(y))| \le
\norm{\vw_i}_1 \norm{g(x)_j - g(y)_j}_\infty \le |x-y|
\]
which completes the proof.
\end{proof}

\begin{lemma}
\label{lem:norm_bound}
Let $g : \reals \to \reals^s$ such that
$\norm{g(x)}_\infty \le 1$ for every $x \in [0,1]$.
Let $\mW \in \reals^{k \times s}$
a random matrix such that every entry is initialized uniformly from
$[-\frac{1}{2s},\frac{1}{2s}]$,
and let $0< b \le \frac{1}{2}$ some bias.
Denote $h(x) := \sigma(\mW g(x)+b)$.
Then $\norm{h(x)}_\infty \le 1$ for every $x \in [0,1]$.
\end{lemma}

\begin{proof}
As before, we denote $\vw_i$ the $i$-th row of $\mW$, then for every
$x \in [0,1]$ we get:
\[
\norm{h(x)}_\infty \le 
\max_i |\vw_i^\top g(x)+b| \le \max_i \norm{\vw_i}_1 \norm{g(x)}_\infty+b
\le 1
\]
\end{proof}

Iteratively applying this lemma gives a bound on the norm of any hidden representation
in the network:
\begin{lemma}
\label{lem:norm_bound_all}
Assume we initialize a neural-network $\net$
as described in \thmref{thm:hardness},
and denote $\net = g^{(t)} \circ \dots \circ g^{(1)}$.
Denote $G^{(t\p)} = g^{(t\p)} \circ \dots \circ g^{(1)}$,
the output of the layer $t\p$.
Then for every layer $t\p$ and for every $x \in [0,1]$ we get that:
$\norm{G^{(t\p)}(x)}_{\infty} \le 1$.
\end{lemma}

\begin{proof}
First, $g^{(1)}(x) = \sigma(\vw^{(1)} x + \vb^{(1)})$,
where $\vw^{(1)} \sim \mathcal{U}([-\frac{1}{2}, \frac{1}{2}]^k)$
and $\vb^{(1)} = [\frac{1}{2}, \dots, \frac{1}{2}]$,
so for every $x \in [0,1]$ we have:
\[
\norm{g^{(1)}(x)}_\infty \le \norm{\vw^{(1)}}_\infty |x| + \frac{1}{2}
 \le 1
\]
Now, from \lemref{lem:norm_bound}, if
$\norm{G^{(t\p)} (x)}_\infty \le 1$ for $x \in [0,1]$,
then
$\norm{G^{(t\p+1)} (x)}_\infty \le 1$ for $x \in [0,1]$.
By induction we get that
$\norm{G^{(t\p)} (x)}_\infty$ for $x \in [0,1]$ for every 
$t\p \le t-1$.
Finally, we have $\vw^{(t)} \sim \mathcal{U}([-\frac{1}{2k}, \frac{1}{2k}])$
and $b^{(t)} = \frac{1}{2}$, and this gives us for every $x \in [0,1]$:
\[
|\net(x)| \le
|\vw^{(t)} G^{(t-1)} (x)+b^{(t)}|
\le \norm{\vw^{(t)}}_1 \norm{G^{(t-1)} (x)}_\infty
+ b^{(t)} \le 1\]
\end{proof}

We also show that the gradients are bounded on all the examples in the distribution:

\begin{lemma}
\label{bounded_grad}
Assume we initialize a neural-network $\net$
as described in \thmref{thm:hardness}. Then for every layer $t\p$,
and every example $x \in [0,1]$ we have:
\begin{align*}
&\norm{\frac{\partial}{\partial \mW^{(t\p)}} \net(x)}_{\max}
\le 1 \\
&\norm{\frac{\partial}{\partial \vb^{(t\p)}} \net(x)}_{\infty}
\le 1 \\
\end{align*}
\end{lemma}

\begin{proof}
Recall that we denote for every $x \in [0,1]$ the output of the layer $t\p$
to be $x^{(t\p)} = G^{(t\p)}$.
Denote $\mD^{(t\p)} = \diag(\sigma\p(\mW^{(t\p)} x^{(t\p)}))$.
We calculate the gradient of the weights at layer $t\p$:
\begin{align*}
\frac{\partial}{\partial \mW^{(t\p)}} \net(x)
&= \frac{\partial}{\partial \mW^{(t\p)}}
g^{(t)} \circ \dots \circ g^{(t\p)} (\sigma(\mW^{(t\p)} x^{(t\p-1)} + \vb^{(t\p)})) \\
&= (\vw^{(t)})^{\top} D^{(t-1)} \mW^{(t-1)} \cdots D^{(t\p+1)} \mW^{(t\p+1)} D^{(t\p)} x^{(t\p-1)}
\end{align*}
Denote $\norm{\cdot}_{\infty}^{OP}$ the operator norm induced by $\ell_\infty$,
and we get (using the properties of the weights initialization):
\begin{align*}
\norm{D^{(t-1)} \mW^{(t-1)} \cdots D^{(t\p+1)} \mW^{(t\p+1)} D^{(t\p)} x^{(t\p-1)}}_\infty
\le&~ \norm{D^{(t-1)}}_{\infty}^{OP} \norm{\mW^{(t-1)}}_{\infty}^{OP}
\cdots \\
&\cdot \norm{D^{(t\p+1)}}_{\infty}^{OP}
\norm{\mW^{(t\p+1)}}_{\infty}^{OP} \norm{D^{(t\p)}}_{\infty}^{OP}
\norm{x^{(t\p-1)}}_\infty \\
&\le 1
\end{align*}
And therefore:
$\norm{\frac{\partial}{\partial \mW^{(t\p)}} \net(x)}_{\max} \le 1$
Finally, we calculate the gradient of the bias at layer $t\p$:
\begin{align*}
\frac{\partial}{\partial \vb^{(t\p)}} \net(x)
&= \frac{\partial}{\partial \vb^{(t\p)}}
g^{(t)} \circ \dots \circ g^{(t\p)} (\sigma(\mW^{(t\p)} x^{(t\p-1)} + \vb^{(t\p)})) \\
&= (\vw^{(t)})^{\top} D^{(t-1)} \mW^{(t-1)} \cdots D^{(t\p+1)} \mW^{(t\p+1)} D^{(t\p)}
\end{align*}
And since $\norm{\vw^{(t)}}_\infty \le 1$ we get similarly to above that
$\norm{\frac{\partial}{\partial \vb^{(t\p)}} \net(x)}_\infty \le 1$.
\end{proof}

\begin{proof} of \thmref{thm:hardness}.
Denote each layer of the network by $g^{(i)}$, so we have:
$\net(x) = g^{(t)} \circ \dots \circ g^{(1)}(x)$.
We show that two things hold on initialization:
\begin{enumerate}
\item $|\net(x)| \le 1$ for $x \in [0,1]$: immediately from \lemref{lem:norm_bound_all}.

\item With probability at least $1-\delta$, for every $j$,
$\net$ is affine on $I_j$:

Denote $\hat{\delta} = \frac{\delta}{t}$, and notice that by the choice
of $n\p$, we get that
$2k^2 \left(\frac{2}{3}\right)^{n\p} \hat{\delta}^{-1} < \frac{1}{2} = b$.
Therefore, since $\sigma$ is affine on all intervals away from zero,
we can apply \lemref{lem:affine} on all the hidden
layers of the network
(choosing $s=1,g=id$ for the first layer and
$s=k, g=g^{(t\p)} \circ g^{(1)}$ for the rest), and use union bound to
get the required.
\end{enumerate}

Now, to prove the theorem, observe that since $\mathcal{D}_n$
is supported on $[0,1]\times \{\pm1\}$, we get that upon initialization
with probability 1 for $(x,y) \sim \mathcal{D}_n$ we have:
$\max \{1-y\net(x),0 \} = 1-y\net(x)$. Since $\net(x)$ is affine
on every $I_j$ w.p $1-\delta$, using \lemref{lem:zero_grad}
we get that in such case:
\begin{align*}
\norm{\frac{\partial}{\partial \tW} \mathcal{L}(\mathcal{N}_{\tW, \mB})}_{\max} &=
\norm{\frac{\partial}{\partial \tW} \mean{(x,y) \sim \mathcal{D}_n}
{\max \{1-y\mathcal{N}_{\tW, \mB}(x),0\}}}_{\max} \\
&= \norm{\mean{(x,y) \sim \mathcal{D}_n}
{-y\frac{\partial}{\partial \tW} \mathcal{N}_{\tW, \mB}(x)}}_{\max} \\
&\le \prob{(x,y) \sim \mathcal{D}_n}{x \in C_{n\p}} \cdot
\norm{\mean{(x,y) \sim \mathcal{D}_n}
{-y\frac{\partial}{\partial \tW} \mathcal{N}_{\tW, \mB}(x) \given x \in C_{n\p}}}_{\max}\\
&+ \prob{(x,y) \sim \mathcal{D}_n}{x \notin C_{n\p}} \cdot
\norm{\mean{(x,y) \sim \mathcal{D}_n}
{-y\frac{\partial}{\partial \tW} \mathcal{N}_{\tW, \mB}(x) \given x \notin C_{n\p}}}_{\max}\\
&\le 4 \left( P(n\p)- \frac{1}{2} \right) + \left( P(n\p)- \frac{1}{2} \right)
= 5 \left( P(n\p)- \frac{1}{2} \right)
\end{align*}
and similarly we get
$\norm{\frac{\partial}{\partial \mB} \mathcal{L}(\mathcal{N}_{\tW, \mB})}_{\infty}
\le 3 \left( P(n\p)- \frac{1}{2} \right)$.

To show that
$\prob{(x,y) \sim \mathcal{D}_n}{\sign(\net(x)) \ne y} \ge
\left(\frac{3}{2} - P(n\p)\right)(1-P(n\p))$,
observe that the $\sign$ function is affine on intervals bounded away
from zero. We can use \lemref{lem:affine} on the final layer,
which shows that $\sign(\net)$ is affine on the intervals $I_j$,
so for every $I_j$ we get either $\sign(\net(I_j)) = \{\hat{y}_j\}$
for some $\hat{y}_j \in \{\pm 1\}$.
Now, using \lemref{lem:zero_grad} we get that:
\begin{align*}
\prob{(x,y) \sim \mathcal{D}_n}{\sign(\net(x)) \ne y \given x \in I_j}
&= \mean{(x,y) \sim \mathcal{D}_n}{\frac{1}{2} - \frac{1}{2}y \hat{y}_j \given x \in I_j}\\
&=\frac{1}{2} - \frac{1}{2}\hat{y}_j\mean{(x,y) \sim \mathcal{D}_n}{y \given x \in I_j}\\
&\ge 1 - P(n\p)
\end{align*}
And from this we get:
\begin{align*}
\prob{(x,y) \sim \mathcal{D}_n}{\sign(\net(x)) \ne y} &= 
\prob{(x,y) \sim \mathcal{D}_n}{x \in C_{n\p}}
\prob{(x,y) \sim \mathcal{D}_n}{\sign(\net(x)) \ne y \given x \in C_{n\p}} \\
&+ \prob{(x,y) \sim \mathcal{D}_n}{x \notin C_{n\p}}
\prob{(x,y) \sim \mathcal{D}_n}{\sign(\net(x)) \ne y \given x \notin C_{n\p}} \\
&\ge \left(\frac{3}{2} - P(n\p)\right)\left(1-P(n\p)\right)
\end{align*}
\end{proof}

\section{Proof of \crlref{crl:hardness_existance}}
\begin{proof}
Denote $a = 2\log^{-1}(\frac{3}{2}),
b = 2\log^{-1}(\frac{3}{2})\log(\frac{4k^2}{\delta})+1$,
and from Lemma A.2 in \cite{shalev2014understanding}, we get that
if $t > 4a \log(2a) + 2b$ then $t > a \log(t)+ b$.
Choosing $n = \frac{t}{2}$ gives $n > \log^{-1}(\frac{3}{2})\log(\frac{4tk^2}{\delta})+1$, so applying \thmref{thm:hardness} shows that
$\mathcal{D}^n_F$ satisfies 3.
\thmref{thm:depth_expressivity} immediately gives 1.
Note that $\mathcal{D}_F^n$ can be realized only by functions with at least
$2^{n-1}+1$ linear regions. Shallow networks on $\reals$ of width $k$
have at most $k+1$ linear regions, so this gives 2.
\end{proof}

\clearpage

\section{Experimental Results}
The following tables summarize the results of all the experiments
that are detailed in Section \ref{sec:experiments}.
\begin{table}[h]
\caption{Performance of different network architectures on various
fractal distributions of depth 5, with different fractal structures.}
\label{experiments_all_fractals}
\vskip 0.15in
\begin{center}
\begin{small}
\begin{sc}
\begin{tabular}{lcccccr}
Depth / Width & 10 & 20 & 50 & 100 & 200 & 400 \\
\multicolumn{6}{l}{Sierpinsky Triangle} \\
1 & 0.78 & 0.82 & 0.88 & 0.87 & 0.89 & 0.90 \\
2 & 0.86 & 0.91 & 0.93 & 0.94 & 0.95 & 0.95 \\
3 & 0.89 & 0.92 & 0.96 & 0.96 & 0.97 & 0.97 \\
4 & 0.87 & 0.94 & 0.97 & 0.97 & 0.97 & 0.98 \\
5 & 0.89 & 0.94 & 0.96 & 0.97 & 0.97 & 0.98 \\
\multicolumn{6}{l}{2D Cantor Set} \\
1 & 0.61 & 0.69 & 0.72 & 0.73 & 0.72 & 0.74 \\
2 & 0.72 & 0.81 & 0.82 & 0.86 & 0.86 & 0.87 \\
3 & 0.78 & 0.84 & 0.88 & 0.92 & 0.93 & 0.93 \\ 
4 & 0.82 & 0.86 & 0.91 & 0.95 & 0.97 & 0.97 \\
5 & 0.81 & 0.87 & 0.95 & 0.97 & 0.99 & 0.98 \\
\multicolumn{6}{l}{Pentaflake} \\
1 & 0.66 & 0.65 & 0.67 & 0.70 & 0.76 & 0.76 \\
2 & 0.71 & 0.73 & 0.79 & 0.81 & 0.82 & 0.83 \\
3 & 0.73 & 0.78 & 0.83 & 0.84 & 0.85 & 0.86 \\
4 & 0.76 & 0.79 & 0.85 & 0.87 & 0.88 & 0.88 \\
5 & 0.76 & 0.81 & 0.86 & 0.88 & 0.87 & 0.90 \\

\multicolumn{6}{l}{Vicsek} \\
1 & 0.59 & 0.60 & 0.63 & 0.66 & 0.67 & 0.68 \\
2 & 0.64 & 0.70 & 0.72 & 0.75 & 0.76 & 0.75 \\
3 & 0.69 & 0.72 & 0.77 & 0.79 & 0.81 & 0.82 \\
4 & 0.71 & 0.74 & 0.79 & 0.82 & 0.83 & 0.84 \\
5 & 0.70 & 0.77 & 0.82 & 0.84 & 0.86 & 0.86 \\
\end{tabular}
\end{sc}
\end{small}
\end{center}
\vskip -0.1in
\end{table}

\begin{table}[h]
\caption{Performance of depth 5 network on the 
different fractal structure (of depth 5),
with varying approximation curves.}
\label{experiments_all_fractals_cruves}
\vskip 0.15in
\begin{center}
\begin{small}
\begin{sc}
\begin{tabular}{lcccccr}
Curve \# / Width & 10 & 20 & 50 & 100 & 200 & 400 \\
\multicolumn{6}{l}{Sierpinsky Triangle} \\
1 & 0.89 & 0.94 & 0.96 & 0.97 & 0.97 & 0.98 \\
2 & 0.89 & 0.94 & 0.96 & 0.97 & 0.97 & 0.97 \\
3 & 0.78 & 0.94 & 0.96 & 0.97 & 0.97 & 0.97 \\
4 & 0.79 & 0.92 & 0.96 & 0.96 & 0.97 & 0.97 \\
5 & 0.76 & 0.89 & 0.96 & 0.97 & 0.97 & 0.97 \\
6 & 0.76 & 0.90 & 0.97 & 0.97 & 0.98 & 0.98 \\
\multicolumn{6}{l}{2D Cantor Set} \\
1 & 0.81 & 0.87 & 0.95 & 0.97 & 0.99 & 0.98 \\
2 & 0.70 & 0.85 & 0.92 & 0.94 & 0.94 & 0.97 \\
3 & 0.62 & 0.73 & 0.75 & 0.80 & 0.91 & 0.89 \\
4 & 0.53 & 0.65 & 0.77 & 0.77 & 0.84 & 0.93 \\
5 & 0.57 & 0.61 & 0.65 & 0.69 & 0.76 & 0.73 \\
6 & 0.53 & 0.64 & 0.66 & 0.78 & 0.71 & 0.61 \\

\multicolumn{6}{l}{Pentaflake} \\
1 & 0.76 & 0.81 & 0.86 & 0.88 & 0.87 & 0.90 \\
2 & 0.59 & 0.68 & 0.77 & 0.78 & 0.80 & 0.84 \\
3 & 0.54 & 0.57 & 0.64 & 0.63 & 0.72 & 0.64 \\
4 & 0.53 & 0.55 & 0.58 & 0.61 & 0.65 & 0.68 \\
5 & 0.52 & 0.52 & 0.52 & 0.55 & 0.60 & 0.57 \\
6 & 0.52 & 0.52 & 0.52 & 0.53 & 0.56 & 0.54 \\

\multicolumn{6}{l}{Vicsek} \\
1 & 0.70 & 0.77 & 0.82 & 0.84 & 0.86 & 0.86 \\
2 & 0.59 & 0.61 & 0.67 & 0.69 & 0.71 & 0.71 \\
3 & 0.56 & 0.55 & 0.58 & 0.64 & 0.64 & 0.65 \\
4 & 0.51 & 0.52 & 0.54 & 0.56 & 0.58 & 0.59 \\
5 & 0.52 & 0.52 & 0.52 & 0.55 & 0.53 & 0.57 \\
6 & 0.53 & 0.51 & 0.51 & 0.55 & 0.58 & 0.59 \\

\end{tabular}
\end{sc}
\end{small}
\end{center}
\vskip -0.1in
\end{table}

\end{document}